\documentclass[bachelorscriptie, english]{uvamath}

\usepackage[english]{babel}
\usepackage{bbding}
\usepackage{graphicx}
\usepackage[pdfborder={0 0 0}]{hyperref}
\usepackage{lipsum}
\usepackage{csquotes}
\usepackage{amsmath,enumitem,wrapfig}
\usepackage{tikz}
\usepackage{graphicx}
\usepackage{enumerate}
\usepackage{amsfonts, amsmath, amssymb, amsthm}
\usepackage{pifont}
\usepackage{skull}
\usepackage{dsfont}

\usepackage{biblatex}

\usepackage{thmtools}
\usepackage{thm-restate}
\usepackage{cleveref}

\newcommand{\shap}{\phi_{\text{SHAP}}}
\DeclareMathOperator*{\argmin}{argmin}
\newcommand{\X}{\mathcal{X}}

\declaretheoremstyle[headpunct=:]{mystyle}

\declaretheorem[style=mystyle, shaded={rulecolor={rgb}{0,0,0}, rulewidth=1pt, bgcolor={rgb}{1,1,1}, margin=3pt}, name=Theorem]{theorem}

\declaretheorem[style=mystyle, shaded={rulecolor={rgb}{0,0,0}, rulewidth=1pt, bgcolor={rgb}{1,1,1}, margin=3pt},name=Lemma, sibling=theorem]{lemma}

\declaretheorem[style=mystyle, shaded={rulecolor={rgb}{0,0,0}, rulewidth=1pt, bgcolor={rgb}{1,1,1}, margin=3pt},name=Claim, sibling=theorem]{claim}

\declaretheorem[style=mystyle, shaded={rulecolor={rgb}{0,0,0}, rulewidth=1pt, bgcolor={rgb}{1,1,1}, margin=3pt},name=Property, sibling=theorem]{property}

\declaretheorem[style=mystyle, shaded={rulecolor={rgb}{0,0,0}, rulewidth=1pt, bgcolor={rgb}{1,1,1}, margin=3pt},style=definition, name=Definition, sibling=theorem]{definition}

\declaretheorem[style=mystyle, shaded={rulecolor={rgb}{0,0,0}, rulewidth=1pt, bgcolor={rgb}{1,1,1}, margin=3pt}, sibling=theorem]{proposition}

\declaretheorem[style=mystyle, shaded={rulecolor={rgb}{0,0,0}, rulewidth=1pt, bgcolor={rgb}{1,1,1}, margin=3pt}, sibling=theorem]{corollary}

\bibliography{references}

\newcommand{\bb}[1]{\mathbb{#1}}

\newcommand{\set}[1]{\{#1\}}

\newcommand{\brac}[1]{\left(#1\right)}
\newcommand{\bric}[1]{\left[#1\right]}

\newcommand{\1}{\mathds{1}}

\newcommand{\0}{\boldsymbol{0}}

\title{Mathematically rigorous proofs for Shapley explanations}
\author[david.van.batenburg@student.uva.nl, 14633485]{David van Batenburg}

\supervisors{dr.\ Tim van Erven}
\secondgrader{dr.\ Krystal Guo}

\coverimage{\includegraphics[scale=0.4]{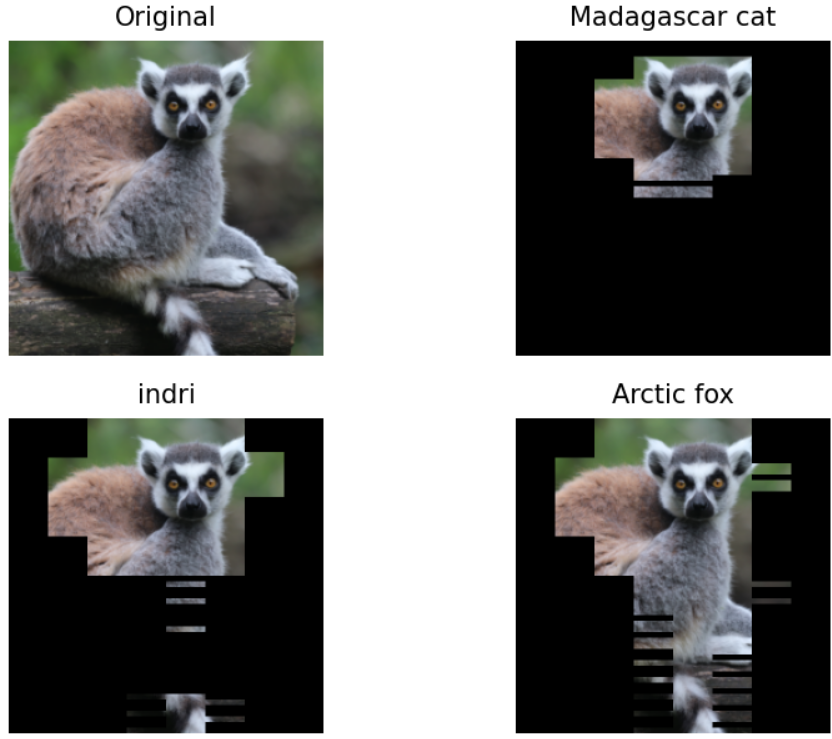}}

\begin{document}
\maketitle
\begin{abstract}
Machine Learning is becoming increasingly more important in today's world. It is therefore very important to provide understanding of the decision-making process of machine-learning models. A popular way to do this is by looking at the Shapley-Values of these models as introduced by Lundberg and Lee. 

In this thesis, we discuss the two main results by Lundberg and Lee from a mathematically rigorous standpoint and provide full proofs, which are not available from the original material.

The first result of this thesis is an axiomatic characterization of the Shapley values in machine learning based on axioms by Young. We show that the Shapley values are the unique explanation to satisfy local accuracy, missingness, symmetry and consistency. Lundberg and Lee claim that the symmetry axiom is redundant for explanations. However, we provide a counterexample that shows the symmetry axiom is in fact essential.

The second result shows that we can write the Shapley values as the unique solution to a weighted linear regression problem. This result is proven with the use of dimensionality reduction.
\end{abstract}

\tableofcontents

\chapter{Introduction}

The rise of AI in today's world makes our lives easier in a lot of ways. We can use image recognition to easily look up what species a certain flower belongs to, just by taking a picture; We can translate text from almost all languages to each other with ease; We can use AI to generate images from prompts. This list goes on and on. AI is really useful, but its rise also comes with new challenges. One of these challenges is that machine learning algorithms are giant black boxes. This means that we, as outside observers, have no idea what goes on inside the algorithm. We do not know why an AI makes certain decisions. ``Why does my image-recognition-algorithm think that this image is a cat and where does it look to conclude this?" is a question that the network does not answer, it only gives the final answer. 

This might seem really innocent. Sure it is fun to know why an algorithm thinks that a certain image is a cat, but in a lot of cases it can be imperative to know why an algorithm makes certain decisions. AI is currently being used in medical image recognition to diagnose patients. In this process, it is very important that a doctor will be able to get a better insight in the diagnostisation made by a model \cite{pinto2023artificial}.

To get an insight about the decision made by a machine learning model, we can use an explanation. These machine learning models are described by a function $f:\mathcal X\to\bb R$ for some $\mathcal X\subseteq\bb R^n$. The outputs of this model will be interpreted in a way that is dependent on the type of model. In the case of a classification model with classes $\set{\pm 1}$, the output $f(x)$ can for example be put in the sign function. The point $x$ can then by assigned the class $\operatorname{sign}(f(x))$. Another example is getting a loan at a bank. In this case, the input of the model might be a list of applicable data, for example current dept, number of children or if the person is living with is parents. The output of the model in this case might be the amount that the person is able to loan from the bank.

For these models, we want to be able to interpret the behavior of a model. In the case of the bank loan, this can be done as follows. The input is a vector $x\in\bb R^n$. We now want to create a vector $\phi(f, x)\in\bb R^n$. This vector is called the explanation of $f$ at $x$. For each $i$, the index $\phi_i(f, x)$ is linked to the index $x_i$ such that $\phi_i(f, x)$ explains how important the feature $x_i$ is in determining the size of a loan. 

A very popular explanation is the SHAP-explanation that is introduced in the paper ``A unified approach to interpreting models" by Lundberg and Lee \cite{NIPS2017_8a20a862}. This paper from 2017 uses game theory to create an explanation for binary classifiers. The popularity of this paper is shown by the fact that it has over 29 thousand citations according to Google Scholar. In the SHAP-explanation, we call $\phi(f, x)$ the Shapley values.

The SHAP-explanation is an axiom-motivated method. This means that Lundberg and Lee introduce four axioms (properties that an explanation can satisfy) and prove that the SHAP-explanation is the unique explanation that satisfies these axioms. The advantage that this approach has is that we do not need to understand the exact explanation to motivate why we want to use this method, we only need to look at the axioms from which we derive it.

Lundberg and Lee prove two important results about the SHAP-method in their paper. The first result is the axiomatic motivation for the SHAP-explanation as mentioned above. Lundberg and Lee show that the SHAP-explanation is the unique explanation to satisfy local accuracy, missingness, symmetry and consistency. They do this by referring to a theorem by Young \cite{young1985monotonic} about the Shapley values in game theory. Lundberg and Lee also claim that the symmetry axiom is actually redundant to show uniqueness of the Shapley values for explanations, because it is implied by another property, consistency.

The second result that Lundberg and Lee show is that the Shapley values are the solution to a weighted linear regression problem. It was already known that the Shapley values is game theory can be seen as the solution to a minimization problem \cite{lehrer2003allocation}. Lundberg and Lee translate this theorem from game theory to machine learning. This theorem is very important for the SHAP-explanation, because it allows us to approximate the Shapley values. To directly calculate the Shapley values takes an infeasable amount of computation, so this approximation makes the Shapley values practical.

While these results are very important, their mathematical completeness can be improved. The first problem is that the proofs that Lundberg and Lee provide are not very rigorous. Take \textit{Theorem 1} from the paper for example. This theorem is not at all trivial and there is no proof to be found for this theorem. \textit{Theorem 2} from Lundberg and Lee has a more rigorous proof, but it is still not complete.

A second problem with the paper by Lundberg and Lee is that the symmetry property is not implied by the consistence property, in contrary to what Lundberg and Lee claim. In this thesis, we will provide a counterexample to show that the symmetry axiom is not redundant and that it is actually necessary to prove the uniqueness of the Shapley values.

\paragraph{Main contributions:} The aim of this thesis is to give a discussion about ``A Unified Approach to Interpreting Models" by Lundberg and Lee and give a formal proof of the theorems from this paper. To do this, we will first discuss the necessary definitions and theorems about explanations and about game theory in \autoref{chapter:2}. The definitions about explanations are based on the definitions given by Lundberg and Lee \cite{NIPS2017_8a20a862} and the definitions and theorems about game theory are based on Young \cite{young1985monotonic}. After this, in \autoref{chapter:discussion}, we discuss the definitions as defined by Lundberg and Lee and give reformulations based on this discussion. With these reformulations, in \autoref{chapter:proof_theorem_1}, we will prove the reformulation of \textit{Theorem 1} from the paper from Lundberg and Lee by making a correspondence between explanations and cooperative games. Finally, in \autoref{chapter:4}, we will prove \textit{Theorem 2} from Lundberg and Lee. This theorem shows that the Shapley values are the solution to a weighted linear regression problem.

\chapter{Explanations and cooperative games}
\label{chapter:2}
In this section, we will review the important literature for this paper. We will first give the definitions as stated in the paper from Lundberg and Lee \cite{NIPS2017_8a20a862}. After this, we will give definitions as described in the paper from Young \cite{young1985monotonic}.

\section{Models and explanations}
\subsection{Simplified models}
In machine learning, we often work with classification functions. First let $k\in\bb N$ and let $\X\subseteq\bb R^k$. A classification function $f:\X\to\bb R$ is a function that classifies each point in $\X$ to some class. An example of this can be seen if we look at a classifier $f:\bb R^k\to\bb R$ where we want to classify a point $x\in\bb R^k$ with classes $\{\pm 1\}$. This can be done by assigning the class $\operatorname{sign}(f(x))$ to the point $x$. 

Given a classifier $f:\X\to\bb R$ and some point $x\in\X$, we want to create an explanation. Let $d\in\bb N$. A local explanation\footnote{In this thesis, we will only look at local explanations, so the terms \textit{explanation} and \textit{local explanation} will be used synonymously.} of $f$ at the point $x$ is a vector $\phi(f, x)\in\bb R^d$ that is dependent on the model $f$ and the point $x$. Usually, $\phi(f, x)$ is defined in a way such that the values in $\phi(f, x)$ can be used to explain why a model gave a certain output. One way to create this link between $\phi(f, x)$ and $x$ is with the use of a simplification function and a simplified model. 

\begin{definition}[Simplification function]
    Let $\mathcal X\subseteq\bb R^n$ and let $d\in\bb N$. Now let $x\in\mathcal X$. We call $h_x:\set{0, 1}^d\to\mathcal X$ the simplification function of $x$ and $x'\in\set{0, 1}^d$ the simplified input of $x$ if the following conditions are satisfied:
    \begin{itemize}
        \item[(a)] $h_x$ is injective;
        \item[(b)] $h_x(x') = x$.
    \end{itemize}
\end{definition}

For a vector $z'\in\set{0, 1}^d$ and an index $i\in\set{1, \dots, d}$, we call $i$ an active index of $z'$ if $z'_i = 1$. We will denote $\mathcal A(z')$ to be the active indices of $z'$. In a more formal definition, for $z'\in\set{0, 1}^d$, we define $$\mathcal A(z'):= \set{i\in[d]: z'_i = 1}.$$

As stated before, the simplification function is used to create a link between $\phi(f, x)$ and $x$ and we wanted this to be interpretable. We therefore also want to define the simplification function in a way that is interpretable. We can do this by linking each index in the input of a simplification function $h_x$ to a certain group of input variables of $f$. We then define $h_x$ such that if we let $z'\in\set{0, 1}^d$ with $i\in\mathcal A(z')$, then $h_x(z')$ would have the input variables that are linked with index $i$ be as they are in $x$. For an index $j$ that is not in $\mathcal A(z')$, there would be some way to not use the input variables in $\X$ that are linked with index $j$. This can be done for example by taking an average over the training data or setting the input variables linked to index $j$ to fixed values.

We also make the assumption that $h_x$ is injective. Since the domain of $h_x$ has $2^d$ elements and the codomain has an uncountable size, most functions will satisfy this condition. 

The ways that $h_x$ can be implemented can be best shown with an example.
\paragraph{Example:}
    Let $f: \mathcal X\to\bb R$ be a model that classifies black and white images. Suppose that the input for this model is in the form of a black-and-white image with a width and height of $256$ pixels. Then The input space will be $\mathcal X = [0, 1]^{256\times 256}$ (assuming that each pixel takes values in $[0, 1]$). We will now divide these pixels into 4 superpixels. We will label these superpixels as follows $$\begin{bmatrix}
        1&2\\3&4
    \end{bmatrix}.$$ Now let $x\in\mathcal X$ be an image. We will define $h_x:\set{0, 1}^4\to\mathcal X$ as the simplified image that sends a vector $z'\in\set{0, 1}^4$ to the image that has the values of $x$ in the superpixels linked to $\mathcal A(z')$ and sets the remaining superpixels to black. The codomain of this $h_x$ is the set of binary vectors with $4$ indices. Each index of a vector $z'\in\set{0, 1}^4$ will be linked to a superpixel. The way that $h_x$ works is shown in \autoref{fig:penguin}. The left image is the point $x\in\mathcal X$. All superpixels (highlighted by the black outpline) are as they are in $x$, so the left image is $h_x(\begin{pmatrix}
        1&1&1&1
    \end{pmatrix})$. The right image has the superpixels linked with index $1$ and $4$ filled in with black. This image is therefore $h_x(\begin{pmatrix}
        0&1&1&0
    \end{pmatrix})$.
    \begin{figure}[h]
        \centering
        \includegraphics[width=0.8\linewidth]{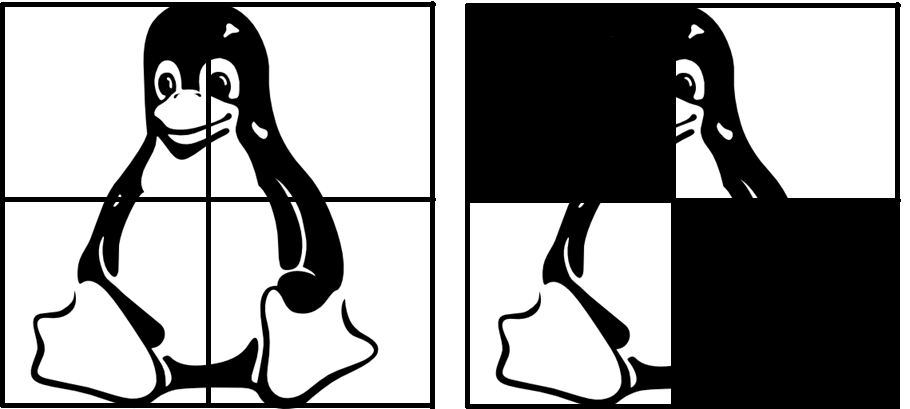}
        \caption{Visualisation of superpixels. The left image is the original image $x$ with the superpixels highlighted. This image will be the output of $h_x(\begin{pmatrix}
            1 & 1 & 1 & 1
        \end{pmatrix})$. The right image is the output $h_x(\begin{pmatrix}
            0 & 1 & 1 & 0
        \end{pmatrix}$).}
        \label{fig:penguin}
    \end{figure}
\\
\\

While it is not always the case, it can be useful for the intuition to view the function $h_x$ as a way to include or exclude certain input variables in a point $x\in\X$.

Now let $S\subseteq[d]$. We define $1_S\in\{0, 1\}^d$ to be the vector such that index $i$ of $1_S$ is 1 if and only if $i\in S$. For example, $d=3$ and $S=\{2, 3\}$ gives us that $$1_S = \begin{pmatrix}
    0\\1\\1
\end{pmatrix}.$$
With this notation, we can define a simplified model.

\begin{definition}[Simplified model]
    Let $f:\X\to\bb R$ be a model, let $x\in\X$ and let $h_x:\set{0, 1}^d\to\X$ be a simplification function. The simplified model $f_x:\mathcal P([d])\to\bb R$ is defined by $$f_x(S) := f(h_x(1_S)).$$
\end{definition}
A simplified model has a lot less information than the model from which it originated, but it does capture all of the information that is obtained from $x$ and the simplification function $h_x$. For certain explanations, this information is enough.

In the following sections, we will assume that the mapping $x\mapsto (h_x, x')$ is fixed. We will also assume that the domain of $h_x$ is $\set{0, 1}^d$. 

This thesis will look at the SHAP-explanation introduced by Lundberg and Lee \cite{NIPS2017_8a20a862}.

\begin{definition}[Shapley values for machine learning]
    Let $f:\mathcal X\to\bb R$ be a model, let $x\in\mathcal X$ and let $h_x:\set{0, 1}^d\to\X$ be the simplification function of $x$. The Shapley values of $f$ in the point $x$ are defined as \begin{equation}
        \shap(f_x)_i = \sum_{\overset{S\subseteq \mathcal A(x')}{i\in S}}\frac{(|S| - 1)!(|\mathcal A(x')| - |S|)!}{|\mathcal A(x')|!}\bric{f_x(S) - f_x(S\setminus\set i)}
    \end{equation}
    for $i\in[d]$.
\end{definition}
In this definition, $\shap$ has the input $f_x$ instead of $(f, x)$, because the Shapley values are only dependent on $f_x$ and not on the behavior of $f$ on values that are not attained by $h_x$.

Intuitively, we can see the value $f_x(S) - f_x(S\setminus\set i)$ as the contribution of index $i$ in the set $S$. We can then see index $i$ of $\shap(f_x)$ as the average contribution of index $i$ over all sets $S\subseteq[d]$ with $i\in S$.

\subsection{Properties}
The Shapley values are the unique explanation that satisfies certain properties. This will be proved in \autoref{thm:true_shap}. In the following section, we will introduce these properties. In the following section, we will let $\X$ be a subset of $\bb R^n$.

\begin{property}[Local Accuracy]
    We say that an explanation $\phi$ satisfies \textit{local accuracy} if for all $x\in\X$ and all models $f:\mathcal X\to\bb R$, $$\phi_0(f, x) + \sum_{i=1}^d\phi_i(f, x) = f(x),
    \quad\text{where $\phi_0(f, x)=f_x(\emptyset)$.}
    $$
    \label{prop:local_accuracy}
\end{property}
 If an explanation $\phi$ satisfies local accuracy, then summing all of the values in $\phi(f, x)$ and $f_x(\emptyset)$ gives the value $f(x)$. This gives the explanation a sort of normalization to ensure that the values in $\phi(f, x)$ have an upper bound and are easier to interpret.

We also want to ensure that if the data-point $x$ does not include certain input variables, then these input variables are also not important in determining $f(x)$. This is given with the following property. 
\begin{property}[Missingness]
    We say that an explanation $\phi$ satisfies \textit{missingness} if $$x'_i = 0\quad\implies\quad \phi_i(f, x) = 0,$$ for all $x\in\X$ and all $f:\mathcal X\to\bb R$.
    \label{prop:missingness}
\end{property}
This property means that if $x$ misses certain data, for example when an experiment made a failed measurement that causes corrupted data, then that missing data is not important in determining $f(x)$.

Furthermore, for a simplification function $h_x:\set{0, 1}^d\to\bb R$, we want that our explanation treats all indices of the input of $h_x$ the same. 
\begin{property}[Symmetry]
    We say that an explanation $\phi$ is \textit{symmetric} if the following implication holds for all $x\in\X$ and all $f:\mathcal X\to\bb R$. For $i, j\in[d]$, if $$f_x(S\cup\set i) = f_x(S\cup\set j)\quad\text{for all $S\subseteq[d]\setminus\set{i, j}$},$$ then $\phi_i(f, x) = \phi_j(f, x)$.
    \label{prop:symmetry_ll}
\end{property} 
This states that if $i$ and $j$ have the same contribution, then the explanation must attribute them the same value.

It is important to note that, in their paper, Lundberg and Lee do not specifically define from what set $S$ is a subset and they do not define from which sets $i$ and $j$ originate. The formulation above is assumed, because it matches with the formulation of the consistency property that is defined below.

\begin{property}[Consistency]
    We say that $\phi$ is \textit{consistent} if for any two models $f,f':\mathcal X\to\bb R$ and all $x\in\mathcal X$, if $$f'_x(S) - f'_x(S\setminus\{i\})\geq f_x(S) - f_x(S\setminus\{i\})\quad\text{for all $S\subseteq[d]$}$$ then $\phi_i(f', x)\geq\phi_i(f, x)$. 
    \label{prop:consistency}
\end{property}
This property means that if for two simplified models $f_x, f'_x$, if an index $i\in[d]$ has a larger contribution for all $S\subseteq[d]$ in $f'_x$ than in $f'_x$, then index $i$ must be more important in determining $f(x)$.

As stated before, the motivation for looking at these properties is given by the Shapley values. Lundberg and Lee give a claim that is very similar to this statement.
\begin{claim}
    \label{claim:ll_shap}
    Let $f:\X\to\bb R$ be a model, let $x\in\X$ and let $h_x:\set{0, 1}^d\to\X$ be a simplification function. There is a unique explanation $\phi(f, x) = \brac{\phi_i(f, x):i\in[d]}$ of $f$ that satisfies Local Accuracy, Missingness and Consistency. 
    For $i\in\{1, \dots, d\}$, this explanation is given by \begin{equation}\phi_i(f, x) =
        \sum_{S\subseteq \mathcal A(x')}\frac{|S|!(d - |S| - 1)!}{d!}\bric{f_x(S) - f_x(S\setminus\{i\})},\label{eq:shap_explanation2}
        \end{equation}
        where we use the convention that if $i\notin S$, then $S\setminus\set i=S$. 
\end{claim}

Lundberg and Lee do not provide a proof for this claim. In \autoref{chapter:discussion}, we will show that this claim is false as currently stated. Furthermore, we will show how the claim can be modified to be correct.

\section{Cooperative games and allocation procedures}
\label{chap:young}
The theorem that the Shapley values are the unique explanation to satisfy the properties listed in the previous section is very similar to a theorem in game theory. This section will introduce this theorem from game theory and give the necessary definitions to understand it. These definitions and the theorem from the following section are all defined and proved by Young \cite{young1985monotonic}.

We will first define a cooperative game. In a cooperative game, we have a set of players that all participate in a game. Their performance in this game gets a score. More formally, we define a cooperative game as follows:
\begin{definition}[Cooperative Game]
    A \textit{cooperative game} with \textit{players} $\{1, \dots, n\} = [n]$ is a function $\nu: \mathcal P([n])\to\bb R$ such that $\nu(\emptyset) = 0$. Here, $\mathcal P([n])$ is used to denote the powerset of the set $[n]$. We call $\nu(S)$ the \textit{value} of $S$ for $S\subseteq[n]$. 
\end{definition}
We can interpret $\nu$ as the function that gives a score to a set of players. Each player may choose if they want to cooperate in the game. If $S$ is the set of players that choose to cooperate, then $\nu(S)$ is the score that they will receive after the game.

\textit{Remark:} In this definition, we might also take any finite set $X$ as our players. This is equivalent, because in a cooperative game, we never use any properties of elements in $[n]$, but only use the finiteness of this set.

\begin{definition}[Allocation Procedure]
    An \textit{allocation procedure} is a funcion $\phi$ that maps a cooperative game $\nu$ to a vector $\phi(\nu)\in\bb R^n$. We denote $\phi_i(\nu)$ to be the $i$-th index of $\phi(\nu)$. We often call $\phi(\nu)$ a \textit{solution concept} to $\nu$. 
\end{definition}
An allocation procedure can be seen as the procedure to attribute a score to each player in a way such that the score is equivalent to the contribution of a player.

One of these allocation procedures is the Shapley value. The definition of the Shapley value is as follows:
\begin{definition}[Shapley values for cooperative games]
    Let $\nu:[n]\to\bb R$ be a cooperative game. The allocation procedure $\phi$ defined by
    \begin{equation}
    \phi_i(\nu) := \sum_{S\subseteq[n]\colon i\in S}\frac{(|S| - 1)! (n - |S|)!}{n!}\nu^i(S) \label{eq:shap}
    \end{equation}
    is called the \textit{Shapley value}. Here $\nu^i(S)$ is the marginal contribution of $i$ in $S$ that is defined as follows \[
\nu^i(S) = \begin{cases}
    \nu(S) - \nu(S\setminus\{i\}),\quad &i\in S\\
    \nu(S\cup\set i) - \nu(S),\quad&i\notin S
\end{cases}
\] for $S\subseteq[n]$.
\end{definition}
The Shapley value of a player $i$ can be seen as the average contribution of player $i$ over all possible combinations of players that include player $i$.

Just like with explanations for machine learning models, the Shapley values are the unique allocation procedure that satisfy certain properties. These properties will be defined below. 

In an allocation procedure, it is often desired that the indices say something about the contribution of their respective players. For the result that we want to use in this paper, Young introduces the following properties. 
\begin{property}[Efficiency]
    We say that $\phi(\nu)$ is \textit{efficient} if \begin{equation}
    \sum_{i=1}^n\phi_i(\nu) = \nu([n]).\label{eq:efficiency}    
    \end{equation}
    \label{prop:efficiency}
\end{property}
This property can be seen as a way to normalize the values of $\psi(\nu)$ to ensure that they stay bounded and sum to $\nu([n])$.

In a cooperative game, we want to treat all players the same, regardless of their position. 
\begin{property}[Symmetry]
    We say that $\phi$ is \textit{symmetric} if for all permutations $\sigma:[n]\to[n]$ we have \begin{equation}\phi_{\sigma (i)}(\sigma(\nu)) = \phi_i(v).\end{equation}
    \label{prop:symmetry2}
\end{property}
The cooperative game $\sigma(\nu)$ is defined as $\sigma(\nu)(S) = \nu(\sigma(S))$.
This property implies that, if we swap the position of our players, then their contributions stay the same, meaning that our allocation procedure treats all players the same. 

\begin{property}[Strong monotonicity]
    We say that $\phi(\nu)$ is strongly monotonic if for any two cooperative games $\nu, \mu:\mathcal P([n])\to\bb R$ and $i\in[n]$, if
    \begin{equation*}
    \nu^i(S)\geq\mu^i(S)\quad\text{for all $S\subseteq[n]$,} \label{eq:monotonicity}
    \end{equation*}
    then $\phi_i(\nu)\geq\phi_i(\mu)$.
    \label{prop:strong_monotonicity}
\end{property}
Strong monotonicity gives us that if some player $i\in[n]$ has a bigger marginal contribution for all possible player-sets $S\subseteq[n]$ in a game $\nu$ than in a game $\mu$, then player $i$ must have a bigger overall contribution to the game $\nu$ than the game $\mu$.

All of the properties listed above can be used in the following theorem:
\begin{theorem}
    The \textit{Shapley value} is the unique allocation procedure that is symmetric, strongly monotonic and efficient.
    \label{thm:shapley}
\end{theorem}
\begin{proof}
    The proof of this theorem is given in Young (1985) \cite{young1985monotonic}.
\end{proof}

\chapter{Discussion on Lundberg and Lee}
\label{chapter:discussion}
In this section, we will discuss the part of the paper by Lundberg and Lee that discusses \autoref{claim:ll_shap}. We will first show that the symmetry property is not redundant by giving a counterexample to the claim, made by Lundberg and Lee, that consistency implies symmetry.
We will also give a discussion about the formulation in the definition of the symmetry and consistency properties for explanations. After this, we will discuss the formulation of \autoref{claim:ll_shap} made by Lundberg and Lee. Finally, we end this section by giving a reformulation of the previously mentioned definitions and the claim.

\section{Counterexample: Consistency does not imply symmetry}
\label{chap:counterexample}
In \cite{NIPS2017_8a20a862}, Lundberg and Lee state that the symmetry property is implied by the consistency property. They give a small proof of this statement, which is false. In this section, we will prove that this statement is false, by giving an explanation that satisfies Local Accuracy, Missingness and Consistency, but does not satisfy symmetry.

As in the previous sections, we will use $\mathcal A(x')$ as the set of active indices of $x'$. Since $\mathcal A(x')$ is finite, we can give an enumeration of this set: $\mathcal A(x') = \set{p_1, p_2, \dots, p_k}$, with $k = |\mathcal A(x')|$. We can now define a sequence of sets $S_{i} := \set{p_1, \dots, p_i}$. We will now define an explanation $\psi$ as follows. Let $j\in[d]$, let $f:\X\to\bb R$ be a model and let $x\in\X$. We define \begin{equation*}
    \psi_{j}(f, x) = \begin{cases}
        f_x(\mathcal A(x')\setminus S_{{i-1}}) - f_x(\mathcal A(x')\setminus S_{i})&\quad \text{if }j = p_i\text{ for some $i\in\set{2, \dots, k}$}\\
        f_x(\mathcal A(x')) - f_x(\mathcal A(x')\setminus S_{1})&\quad \text{if }j = p_1\\
        0&\quad \text{if }j\notin \mathcal A(x')
    \end{cases}
\end{equation*}
In the following lemmas, we will let $\psi$ be the explanation defined above.
\begin{lemma}
    $\psi$ satisfies local accuracy.
\end{lemma}
\begin{proof}
    Let $f:\X\to\bb R$ be a model and let $x\in\X$. We will use the property that $S_{i}\cup \set{{i+1}} = S_{{i+1}}$ for $i\in\set{1, \dots, k-1}$.
    We have that \begin{align*}
        \sum_{i=1}^d\psi_i(f, x) &= f_x(\mathcal A(x')) - f_x(\mathcal A(x')\setminus S_{1}) + \sum_{i=2}^k\brac{f_x(\mathcal A(x')\setminus S_{{i-1}}) - f_x(\mathcal A(x')\setminus S_{i})}\\
        &= f_x(\mathcal A(x')) - f_x(\mathcal A(x')\setminus S_{1}) + f_x(\mathcal A(x')\setminus S_{1}) -  f_x(\mathcal A(x')\setminus S_{k})\\
        &= f_x(\mathcal A(x')) - f_x(\emptyset),
    \end{align*}
    where we make use of the fact that $S_{k}=\mathcal A(x')$ and therefore $\mathcal A(x')\setminus S_{k} = \emptyset$.
    From this, we get that $$\psi_0(f, x) +\sum_{i=1}^d\psi_i(f, x) = f(x).$$
    We conclude that $\psi$ satisfies local accuracy.
\end{proof}
\begin{lemma}
    $\psi$ satisfies missingness.
\end{lemma}
\begin{proof}
    This follows directly from the definition of $\psi$.
\end{proof}
\begin{lemma}
    $\psi$ satisfies strong consistency.
\end{lemma}
\begin{proof}
    Let $f, f':\X\to\bb R$ be models, let $x\in \X$ and let $j\in [d]$. If $j\notin \mathcal A(x')$, then missingness implies that $\psi_j(f', x)\geq\psi_j(f, x)$, since both $\psi_j(f', x)$ and $\psi_j(f, x)$ are $0$. 
    
    Now suppose that $j\in \mathcal A(x')$. There exists $i\in[k]$ such that $j = p_i$. Now suppose that $$f'_x(S) - f'_x(S\setminus\{p_i\})\geq f_x(S) - f_x(S\setminus\{p_i\}),\quad\text{for all } S\subseteq[d].$$ Now suppose that $i\neq 1$. We have that $(\mathcal A(x')\setminus S_{{i-1}})\setminus\set{p_i} = \mathcal A(x')\setminus S_{i}$. Choosing $S=\mathcal A(x')\setminus S_{{i-1}}$ in our assumption now gives us that \begin{align*}
        f'_x(\mathcal A(x')\setminus S_{{i-1}}) - f'_x(\mathcal A(x')\setminus S_{{i}})&\geq f_x(\mathcal A(x')\setminus S_{{i-1}}) - f_x(\mathcal A(x')\setminus S_{{i}}).
    \end{align*} Using the definition of $\psi$, we see that this is equivalent to saying that $\psi_{i}(f', x)\geq \psi_{i}(f, x)$.

    Finally, if $i=1$, we have that $S_{1} = \set{p_1}$. Choosing $S=\mathcal A(x')$ in the assumption now gives us that $$f'_x(\mathcal A(x')) - f'_x(\mathcal A(x')\setminus S_{{1}})\geq f_x(\mathcal A(x')) - f_x(\mathcal A(x')\setminus S_{{1}}),$$ which again is equal to saying that $\psi_{p_1}(f', x)\geq \psi_{p_1}(f, x)$.

    We conclude that $\psi_j(f', x)\geq \psi_j(f, x)$ for all $j\in[d]$, so $\psi$ satisfies strong consistency.
\end{proof}

We have now found an explanation that satisfies local accuracy, missingness and consistency. We will now show that this explanation does not satisfy symmetry. 

\begin{lemma}
    $\psi$ does not satisfy symmetry.
\end{lemma}
\begin{proof}
    Let $d=2$ and let $x\in \X$. Let $h_x:\set{0, 1}^2\to\bb R$ be a simplification function such that $x'=\begin{pmatrix}
        1&1
    \end{pmatrix}^T$. We now have that $\mathcal A(x')=\set{1, 2}$. We can now define $p_1=1$ and $p_2=2$. This means that $S_{1} = \set{1}$ and $S_2=\set{1, 2}$. Now define a model $f:\X\to\bb R$ such that \begin{align*}
        f_x(\set{1, 2}) &= 3\\
        f_x(\set 1) = f_x(\set 2) &= 1\\
        f_x(\emptyset) &= 0.
    \end{align*}
    We can do this, because $h_x$ is injective.
    Now let $i=1, j=2$. Then for all $S\subseteq\set{1, 2}\setminus\set{1, 2}$ (meaning $S=\emptyset)$), we have that $$f_x(S\cup\set1) =f_x(S\cup\set 2).$$ We also find that\begin{align*}
        \psi_1(f, x) = f_x(\mathcal A(x')) - f_x(\mathcal A(x')\setminus S_1) = f_x(\set{1, 2}) - f_x(\set 2) &= 2,\\
        \psi_2(f, x) = f_x(\mathcal A(x')\setminus S_{1}) - f_x(\mathcal A(x')\setminus S_{2}) = f_x(\set2) - f_x(\emptyset) &= 1.
    \end{align*}
    Since $\psi_1(f, x)\neq \psi_2(f, x)$, we conclude that $\psi$ does not satisfy symmetry.
\end{proof}

The mistake that Lundberg and Lee made is right before Equation 9 in their proof \cite{NIPS2017_8a20a862}. They claim that swapping $i$ and $j$ will give the required result, but one can verify that doing so actually requires the symmetry axiom to reach the result in Equation 9.

\label{chap:discussion_ll}
\section{Symmetry}
In this section, we will discuss the formulation of the symmetry property by Lundberg and Lee. This formulation is given in their supplementary material. This definition is not very complete, so in this thesis, the details are assumed from their formulation and from the formulation of the consistency property, which Lundberg and Lee do define with more detail.

The problem with this definition is illustrated by the following lemma:
\begin{lemma}
    There is no explanation that satisfies local accuracy, symmetry and missingness.
\end{lemma}
\begin{proof}
    Suppose that $\phi$ is an explanation that satisfies local accuracy, symmetry and missingness. Now let $x\in\X$ and let $h_x$ be a simplification function such that $x'=\begin{pmatrix}
        1&0
    \end{pmatrix}$. We know that $h_x$ is injective, since it is a simplification function. Now define a model $f:\mathcal X\to\bb R$ such that \begin{align*}
        f_x(\set{1, 2}) &= 2\\
        f_x(\set 1) = f_x(\set 2) &= 1\\
        f_x(\emptyset) &= 0.
    \end{align*}
    Since $x'_2=0$, we must have that $\phi_2(f, x)=0$. Symmetry now gives us that $\phi_1(f, x)=0$. We now have that $$\phi_0(f, x) + \phi_1(f, x) + \phi_2(f, x) = 0,$$ which is a contradiction to the assumption that $\phi$ satisfies local accuracy. 
\end{proof}

As we will see in a later section, we do want to have a symmetry definition, but we need to provide a different formulation than the one implied by Lundberg and Lee. The main issue here is that we look at all $S\subseteq[d]$ and $i, j\in[d]$. In \autoref{chapter:reformulation}, we will give a reformulation of this definition that eliminates this problem.

\section{Consistency}
The formulation of the consistency property by Lundberg and Lee can also be discussed. A problem with this property is caused by the condition. Lundberg and Lee give a condition about all $S\subseteq [d]$. This condition is in contradiction with the philosophy of Lundberg and Lee that all indices $i\in[d]\setminus\mathcal A(x')$ should not be important. It is therefore illogical to set a condition on the behaviour of a model $f$ on points that are not important.

\section{Values of the Shapley values}
We will now look at the Shapley values that Lundberg and Lee give in \autoref{claim:ll_shap}. Suppose that $d=3$ and $x'=\begin{pmatrix}
    1&1&0
\end{pmatrix}^T$. Now take any injective $h_x$ and define a model $f$ such that $$f_x(S) = \begin{cases}
    0,&\quad S = \emptyset\\
    1,&\quad \text{$S\neq \mathcal A(x')$ and $S\neq\emptyset$}\\
    2,&\quad S = \mathcal A(x').
\end{cases}$$The Shapley values, as formulated by Lundberg and Lee, are now given by \begin{align*}
    \phi_1(f, x)& = \sum_{S\subseteq\set{1, 2}}\frac{|S|!(2 - |S|)!}{6}[f_x(S) - f_x(S\setminus\set 1)] = \frac{1}{2},\\
    \phi_2(f, x) &= \sum_{S\subseteq\set{1, 2}}\frac{|S|!(2 - |S|)!}{6}[f_x(S) - f_x(S\setminus\set 2)] = \frac{1}{2},\\
    \phi_3(f, x) &= 0.
\end{align*}
We now have that $\phi_0(f, x) + \phi_1(f, x) + \phi_2(f, x) + \phi_3(f, x) = 1\neq 2$. We can conclude that the Shapley values as defined by Lundberg and Lee do not satisfy local accuracy, and therefore that \autoref{claim:ll_shap} is false.

\section{Reformulation}
\label{chapter:reformulation}
Because of the arguments given in the previous above, we will discuss a reformulation of certain the properties defined by Lundberg and Lee. In further sections, we will also see that the Shapley values are the unique explanation to satisfy a combination of these reformulated properties and properties defined by Lundberg and Lee.

The first new property is a reformulation of the symmetry property. This is defined as follows:
\begin{property}[Restricted Symmetry]
    We say that an explanation $\phi$ satisfies \textit{restricted symmetry} if the following implication holds for all $f:\mathcal X\to\bb R$ and all $x\in\mathcal X$. For $i, j\in \mathcal A(x')$, if $$f_x(S\cup\set i)=f_x(S\cup\set j)\quad\text{for all $S\subseteq \mathcal A(x')\setminus\set{i, j}$,}$$ then $\phi_i(f, x) = \phi_j(f, x)$.
    \label{prop:res_symmetry}
\end{property}
The main difference between symmetry and restricted symmetry for explanations is that $S$, $i$ and $j$ are only related to $\mathcal A(x')$ instead of $[d]$.

The second property that we will define is a reformulation of the consistency property. This is defined as follows:
\begin{property}[Restricted Consistency]
    We say that an explanation $\phi$ is \textit{consistent} if for all models $f, f':\mathcal X\to\bb R$ and all $x\in\mathcal X$, if $$f'_x(S) - f'_x(S\setminus\{i\})\geq f_x(S) - f_x(S\setminus\{i\}),\quad\text{for all } S\subseteq \mathcal A(x'),$$ then $\phi_i(f', x)\geq\phi_i(f, x)$. 
    \label{prop:res_consistency}
\end{property}
The difference between consistency and restricted consistency is that the condition of restricted consistency only needs to hold for all $S\subseteq \mathcal A(x')$ instead of all $S\subseteq[d]$.

The term `restricted' in these definitions refers to the fact that we restrict $S, i$ and $j$ to the set $\mathcal A(x')\subseteq[d]$.

These two reformulations, we can look at a theorem that is very similar to \autoref{claim:ll_shap}.
\begin{restatable}{theorem}{trueshap}
\label{thm:true_shap}
Let $f:\mathcal X\to\bb R$ be a model, let $x\in\mathcal X$ and let $h_x:\set{0, 1}^d\to\X$ be the simplification function corresponding to $x$.  There is a unique explanation $\shap(f_x) = \brac{\shap(f_x)_i:i\in[d]}$ of $f$ that satisfies \textit{\hyperref[prop:local_accuracy]{local accuracy}, \hyperref[prop:missingness]{missingness}, \hyperref[prop:res_symmetry]{restricted symmetry} and \hyperref[prop:res_consistency]{restricted consistency}}.
    For $i\in\{1, \dots, d\}$, this explanation is given by \begin{equation}\shap(f_x)_i =
        \sum_{\overset{S\subseteq \mathcal A(x')}{i\in S}}\frac{(|S| - 1)!(|\mathcal A(x')| - |S|)!}{|\mathcal A(x')|!}\bric{f_x(S) - f_x(S\setminus\{i\})}.\label{eq:shap_explanation}
        \end{equation}
\end{restatable}

We use the convention that if $i\notin S$, then $S\setminus\set i=S$.
This theorem is very similar to \autoref{claim:ll_shap}. The differences are the following. The conditions include restricted symmetry and restricted consistency instead of consistency. The Shapley value is also changed. The new value was determined by looking at other literature about the Shapley values in machine learning \cite{teneggi2022shap, aas2021explaining}. This theorem says that the Shapley values are the only explanation that satisfies local accuracy, missingness, restricted symmetry and restricted consistency.

\chapter{Game-theoretic characterisation}
\label{chapter:proof_theorem_1}
In this chapter, we will give a proof of \autoref{thm:true_shap}.
We will prove this theorem by creating a cooperative game from a machine learning model and using \autoref{thm:shapley} from Young \cite{young1985monotonic}. In this chapter, we will assume that $\mathcal X\subseteq\bb R^n$ is a fixed subset and that the mapping $x\mapsto(h_x, x')$ for $x\in\mathcal X$ is fixed.

\section{Induced models and cooperative games}
Let $f:\X\to\bb R$ be a model and let $x\in\X$. We can use this model to define a cooperative game. The players of this game will be $\mathcal A(x')$ and the game will be defined as $\nu_{f_x}:\mathcal P(\mathcal A(x'))\to\bb R$ with \begin{equation}
    \nu_{f_x}(S) = f_x(S) - f_x(\emptyset).
\end{equation} We can see that this is indeed a cooperative game, because $\nu_{f_x}(\emptyset) = 0$ per definition. We will call $\nu_{f_x}$ the cooperative game induced from $f_x$. With this definition, we have defined a cooperative game from a model.

We have now induced a cooperative game from a model, but we also want to do the inverse. To do this, let $x\in\X$ and let $\nu:\mathcal P(\mathcal A(x'))\to\bb R$ be a cooperative game. We can now define a model $f^\nu$ as follows:
$$f^\nu(y) = \begin{cases}
    \nu(S),\quad &\text{$y =h_x(1_S)$ for some $S\subseteq\mathcal A(x')$}\\
    0, \quad &\text{otherwise},
\end{cases}$$
for $y\in\X$.
Since $h_x$ is injective, we see that this construction is well defined. 
We will call the model $f^\nu$ the model induced from $\nu$.
We can see that, since $\nu(\emptyset) = 0$, that $\nu_{f_x^\nu} = \nu$. Furthermore, we have that for all $S\subseteq\mathcal A(x')$ that $$f_x(S) = f^{\nu_{f_x}}_x(S).$$

We finally want to create a correspondence between explanations and allocation procedures. Let $\psi$ be an allocation procedure. We can use this allocation procedure to define an explanation $\phi$ that satisfies missingness as follows: $$\phi_i(f, x) = \begin{cases}
    \psi(\nu_{f_x})&\quad\text{if $i\in\mathcal A(x')$}\\
    0&\quad\text{otherwise}.
\end{cases}$$

Now suppose that $\phi$ is an explanation. We cannot immediately define an allocation procedure from $\phi$. To do this, we need to introduce a new definition. For this definition, we also need to define an equivalence relation. Let $f, f':\mathcal X\to\bb R$ be models and let $x\in\mathcal X$. We will say $f_x\sim f'_x$ if $\nu_{f_x}(S) = \nu_{f'_x}(S)$ for all $S\subseteq \mathcal A(x')$.
\begin{property}[Constant on inducing]
    Let $\phi$ be an explanation. We call $\phi$ \textit{constant on inducing} if the following implication holds for all models $f, f':\mathcal X\to\bb R$ and all $x\in\mathcal X$. If $f_x\sim f'_x$, then $\phi(f, x) = \phi(f', x)$.
    \label{prop:constant_on_inducing}
\end{property}
This property means that for an explanation $\phi$ and a model $f$, that $\phi(f, x)$ is solely determined by the cooperative game $\nu_{f_x}$. 

Now suppose that $\phi$ is constant on inducing. Using this property, we can define an allocation procedure $\psi^x$ on cooperative games with players $\mathcal A(x')$ for some $x\in\X$, with $x'$ the simplified input of $x$, as follows: $$\psi_i^x(\nu) = \phi_i(f^\nu, x)\quad\text{for $i\in\mathcal A(x')$}.$$ Because we assumed that $\phi$ is constant on inducing, we get that for all models $f:\X\to\bb R$ and all $x\in\X$ that $\psi^x_i(\nu_{f_x})=\phi_i(f, x)$ for $i\in\mathcal A(x')$.

Note that we write a superscript $x$ in $\psi^x$. This is, because $\phi(f, x)$ might still have a dependency on $x$.

\section{Correspondence of properties}
In the following section, we will prove equivalence of the properties of allocation procedures and the properties of explanations. In this section, we will assume that $\X\subseteq\bb R^n$. We will also assume that the map $x\mapsto(h_x, x')$ is fixed. Finally, given an explanation $\phi$ that is constant on inducing, we will denote $\psi^x$ to be the allocation procedure obtained from $\phi$ as defined in the previous section ($\psi^x_i(\nu)=\phi_i(f^\nu, x)$.

We will give some lemmas that give a correspondence of the properties of cooperative games and allocation procedures, and models and explanations.
\begin{lemma}
    Let $\phi$ be an explanation that satisfies restricted consistency. Then $\phi$ is constant on inducing.
    \label{lemma:res_consistency->constant_on_inducing}
\end{lemma}
\begin{proof}
    For this proof, we will make use of an observation made by Young \cite[70]{young1985monotonic}. This observation is that if $\phi$ satisfies restricted consistency, then the following implication holds. Let $f, f':\mathcal X\to\bb R$ be two models and let $x\in\mathcal X$. We now have that if $$\text{$f'_x(S) - f'_x(S\setminus\set i)= f_x(S) - f_x(S\setminus\set i)$ for all $S\subseteq \mathcal A(x')$},$$ then $\phi_i(f', x) = \phi_i(f, x).$ This statement follows directly from the definition of $\phi$ satisfying restricted consistency.
    
    Now assume that $f_x\sim f'_x$. This means that $\nu_{f_x}(S) = \nu_{f'_x}(S)$ for all $S\subseteq \mathcal A(x')$. This means that $f_x(S) - f_x(\emptyset) = f'_x(S) - f_x'(\emptyset)$ for all $S\subseteq \mathcal A(x')$. We now have that \begin{align*}
        f_x(S) - f_x(S\setminus\set i) &= f_x(S) - f_x(\emptyset) - f_x(S\setminus\set i) + f_x(\emptyset)\\
        &= f'_x(S) - f'_x(\emptyset) - f'_x(S\setminus\set i) + f'_x(\emptyset)\\
        &= f'_x(S) - f'_x(S\setminus\set i),
    \end{align*} for all $S\subseteq \mathcal A(x')$. From this, we can conclude that $\phi_i(f, x) = \phi_i(f', x)$. We have now proven that $\phi$ is constant on inducing.
\end{proof}

\begin{lemma}
    Let $\phi$ be an explanation that is constant on inducing and satisfies missingness. The following equivalence holds: $$\text{$\phi$ satisfies local accuracy}\quad\iff\quad\text{$\psi^x$ is efficient for all $x\in\X$.}$$
    \label{lemm:loc<->eff}
\end{lemma}
\begin{proof}
    \emph{$\Rightarrow$}: First suppose that $\phi$ satisfies local accuracy. Let $x\in\mathcal X$ and let $\nu:\mathcal P(\mathcal A(x'))\to\bb R$ be a cooperative game. We have \begin{align*}
    \nu(\mathcal A(x')) = \nu_{f^\nu}(\mathcal A(x')) &= f^\nu_x(\mathcal A(x')) - f^\nu_x(\emptyset) \\
    &= f^\nu(x) - f^\nu_x(\emptyset) \\
    &= -f^\nu_x(\emptyset) + \phi_0(f^\nu, x) + \sum_{i=1}^d\phi_i(f^\nu, x)\\
    &= \sum_{i=1}^d\phi_i(f^\nu, x).
    \intertext{From missingness, we get that }
    \sum_{i=1}^d\phi_i(f^\nu, x)&= \sum_{i\in \mathcal A(x')}\phi_i(f^\nu, x)\\
    &= \sum_{i\in \mathcal A(x')}\psi^x_i(\nu)
    \end{align*} 
    We can conclude that $\psi^x$ is efficient for all $x\in\X$.
    
    \emph{$\Leftarrow$}: Assume that $\psi^x$ is efficient for all $x\in\X$. Let $f:\mathcal X\to\bb R$ be a model and let $x\in\mathcal X$. We have \begin{align*}
        \phi_0(f, x) + \sum_{i=1}^d\phi_i(f, x) &= f_x(\emptyset) + \sum_{i\in \mathcal A(x')}\phi_i(f, x)\\
        &= f_x(\emptyset) + \sum_{i\in \mathcal A(x')}\psi^x_i(\nu)\\
        &= f_x(\emptyset) + \nu_{f_x}(\mathcal A(x'))\\
        &= f_x(\emptyset) + f_x(\mathcal A(x')) - f_x(\emptyset)\\
        &= f(x).
    \end{align*}
    Here we use the fact that $f_x(\mathcal A(x')) = f(x)$. We conclude that $\phi$ satisfies local accuracy.
\end{proof}

For the next proof, we will make use of an intermediate lemma. This lemma is useful to simplify proofs by eliminating case-distinctions.
\begin{restatable}{lemma}{equiv}
    Let $\nu, \mu:\mathcal P(\mathcal A(x'))\to\bb R$ be cooperative games. The following are equivalent\begin{itemize}
        \item[(i)] $\nu^i(S)\geq\mu^i(S)$ for all $S\subseteq \mathcal A(x')$;
        \item[(ii)] $\nu(S) - \nu(S\setminus\set i)\geq\mu(S) - \mu(S\setminus \set i)$ for all $S\subseteq \mathcal A(x')$ with $i\in S$.
    \end{itemize} 
    \label{lemma:equiv}
\end{restatable}
\begin{proof}
    \emph{(i)$\Rightarrow$(ii)}: Suppose that (i) holds. We get that for all $S\subseteq \mathcal A(x')$ with $i\in S$ that $$\nu^i(S) = \nu(S) - \nu(S\setminus\set i),\quad\mu^i(S) = \mu(S) - \mu(S\setminus\set i).$$ We now see that $$\nu(S) - \nu(S\setminus\set i)\geq\mu(S) - \mu(S\setminus \set i)$$ for all $S\subseteq X$ with $i\in S$. 

    \emph{(ii)$\Rightarrow$(i):} Suppose that (ii) holds. Now take any $S\subseteq \mathcal A(x')$. If $i\in S$, then we get that $\nu^i(S) = \nu(S)- \nu(S\setminus\set i)$ (similarly with $\mu$). This means that $$\nu^i(S)\geq \mu^i(S).$$ Now assume that $i\notin S$. Then we can define $D=S\cup\set i$. From (ii) we get that $$\nu(D)-\nu(D\setminus \set i)\geq\mu(D)-\mu(D\setminus\set i).$$ Since $D\setminus\set i=S$, we get that $$\nu(S\cup\set i) - \nu(S)\geq\mu(S\cup\set i) - \mu(S).$$ Now using the definition of $\nu^i(S)$ and $\mu^i(S)$ gives us that $$\nu^i(S)\geq\mu ^i(S).$$ With this, the lemma is proven.
\end{proof}
We will now make the link between restricted consistency and strong monotonicity.
\begin{lemma}
    For an explanation $\phi$ that is constant on inducing, we have $$\text{$\phi$ satisfies restricted consistency}\iff\text{$\psi^x$ satisfies strong monotonicity for all $x\in\X$.}$$
    \label{lemma:restricted_consistency->strong_monotonicity}
\end{lemma}
\begin{proof}
    "$\Rightarrow$": Let $\phi$ be an explanation that is constant on inducing and satisfies restricted consistency. Let $x\in\X$, let $\nu, \mu$ be cooperative games with players $\mathcal A(x')$ and let $i\in\mathcal A(x')$. Suppose that for all $S\subseteq \mathcal A(x')$, we have $\nu^i(S)\geq\mu^i(S)$. \autoref{lemma:equiv} now states that this is equivalent to $$\text{$\nu(S) - \nu(S\setminus\set i)\geq\mu(S) - \mu(S\setminus\set i)$ for all $S\subseteq \mathcal A(x')$ such that $i\in S$}.$$ 
    
    We will now make use of this second statement. We get that, for all $S\subseteq \mathcal A(x')$ such that $i\in S$, we have $$f^\nu_x(S) - f^\nu_x(S\setminus\set i)\geq f^\mu_x(S) - f^\mu_x(S\setminus\set i).$$ Combining this with the fact that for all $S\subseteq \mathcal A(x')\setminus\set i$ we have $S\setminus\set i=S$, we get that $$f^\nu_x(S) - f^\nu_x(S\setminus\set i)\geq f^\mu_x(S) - f^\mu_x(S\setminus\set i)\quad\text{ for all }S\subseteq \mathcal A(x').$$ 
    Restricted consistency now implies that $\phi_i(f^\nu,x)\geq \phi_i(f^\mu,x)$, which is equivalent to $\psi^x_{ i}(\nu)\geq\psi^x_{i}(\mu)$. We conclude that $\psi^x$ satisfies strong monotonicity.

    "$\Leftarrow$": Suppose that $\psi^x$ satisfies strong monotonicity for all $x\in\X$. Now let $f, f':\X\to\bb R$ be two models, let $x\in\X$ and let $i\in \mathcal A(x')$. Suppose that $$f'_x(S) - f'_x(S\setminus\set i)\geq f_x(S) - f_x(S\setminus\set i),\quad\text{for all $S\subseteq \mathcal A(x')$.}$$ From this, we get that\begin{align*}
        f_x'(S) - f'_x(\emptyset) - f'_x(S\setminus\set i) + f'_x(\emptyset) &\geq f_x(S) - f_x(\emptyset) - f_x(S\setminus\set i) + f_x(\emptyset)\\
        \nu_{f'_x}(S) - \nu_{f'_x}(S\setminus\set i)&\geq \nu_{f_x}(S) - \nu_{f_x}(S\setminus\set i)
    \end{align*}
    for all $S\subseteq \mathcal A(x')$. From \autoref{lemma:equiv}, we get that $\nu_{f'_x}^i(S)\geq \nu^i_{f_x}(S)$ for all $S\subseteq \mathcal A(x')$. From the fact that $\psi$ satisfies strong monotonicity, we get that $\psi_i^x(\nu_{f'_x})\geq\psi_i^x(\nu_{f_x})$. Because $\phi$ is constant on inducing, we have that $\phi_i(f, x) = \psi_i^x(\nu_{f_x})$, so we can conclude that $\phi_i(f'_x) \geq \phi_i(f_x)$, so $\phi$ satisfies restricted consistency.
\end{proof}

We now want to prove the equivalence of symmetry axioms. To do this, we will need to define a new property.

\begin{property}[New Symmetry]
    Let $n\in\bb N$ and let $\psi$ be an allocation procedure with players $[n]$. We say that $\psi$ satisfies \textit{new symmetry} if the following implication holds. Let $i, j\in [n]$. If $$\nu(S\cup\set i) = \nu(S\cup\set j)\quad\text{for all $S\subseteq [n]\setminus\set{i, j}$},$$ then $\psi_i(\nu) = \psi_j(\nu)$.
\end{property}
This property is very similar to the symmetry property by Lundberg and Lee and is also used in literature about game theory \cite{winter2002shapley}. In the literature, this property is also referred to as symmetry. 

In \autoref{appendix:symmetry}, we prove the following lemma: 
\begin{lemma}
    Let $\psi$ be an allocation procedure that satisfies strong monotonicity. We have $$\text{$\psi$ is symmetric}\quad\iff\quad\text{$\psi$ is newly symmetric.}$$
    \label{lemma:new_symmetry<->symmetry}
\end{lemma}
Using this lemma, we can prove an equivalence of definitions for explanations and allocation procedures.
\begin{lemma}
    Let $\phi$ be an explanation that is constant on inducing and satisfies restricted consistency. We have $$\text{$\phi$ satisfies restricted symmetry} \quad\iff\quad\text{$\psi^x$ satisfies symmetry for all $x\in\X$.}$$
    \label{lemma:restricted_symmetry->symmetry}
\end{lemma}
\begin{proof}
    \emph{$\Rightarrow$}: Let $x\in\X$ and suppose that $\phi$ satisfies restricted symmetry. We will prove that $\psi^x$ satisfies new symmetry. Let $\nu:\mathcal P(\mathcal A(x'))\to\bb R$ be a cooperative game and take $i, j\in \mathcal A(x')$. 
    Now suppose that $$\nu(S\cup\set i) = \nu(S\cup\set j)\quad\text{for all $S\subseteq\mathcal A(x')\setminus\set{i, j}$.}$$
    
    Using the definition of $f^\nu_x$, we get that $$f^\nu_x(S\cup \set i) = f^\nu_x(S\cup\set j) \quad\text{for all $S\subseteq \mathcal A(x')\setminus\set{i, j}$}.$$ From the assumption that $\phi$ satisfies restricted symmetry, we get that $\phi_i(f^\nu, x) = \phi_j(f^\nu, x)$. The definition of $\psi^x$ gives us that $\psi^x_i(\nu) = \psi^x_j(\nu)$. With this, we have proven that $\psi^x$ is newly symmetric. Since $\phi$ satisfies restricted consistency, we can say that $\psi^x$ satisfies strong monotonicity (\autoref{lemma:restricted_consistency->strong_monotonicity}). From \autoref{lemma:new_symmetry<->symmetry}, we can conclude that $\psi^x$ is symmetric.

    \emph{$\Leftarrow$}: Suppose that $\psi^x$ satisfies symmetry for all $x\in\X$. \autoref{lemma:new_symmetry<->symmetry} says that $\psi^x$ also satisfies new symmetry for all $x\in\X$. Now let $f:\X\to\bb R$ be a model, let $x\in\X$ and let $i, j\in \mathcal A(x')$. Suppose that $$f_x(S\cup\set i) = f_x(S\cup\set j),\quad\text{ for all $S\subseteq \mathcal A(x')\setminus\set {i, j}$.}$$ Adding $f_x(\emptyset)$ to both sides gives us that $$\nu_{f_x}(S\cup\set i) = \nu_{f_x}(S\cup\set j),\quad\text{for all $S\subseteq \mathcal A(x')\setminus\set{i, j}$.}$$ Because $\psi^x$ is newly symmetric, we get that $\psi^x_i(\nu_{f_x}) = \psi^x_j(\nu_{f_x})$. From this we can conclude that $\phi_i(f, x) = \phi_j(f, x)$, so we have that $\phi$ satisfies restricted symmetry.
\end{proof}

\section{Uniqueness of the Shapley values for models}
In this section we will prove the first important result from this thesis. We will prove that the axiomatic motivation for the use of the Shapley values is valid with the following theorem.
\trueshap*
\begin{proof}
        This proof will be split into two parts: existence and uniqueness. With existence, we will prove that the Shapley values actually satisfy the required properties and with uniqueness we will prove that the Shapley values are the unique explanation to satisfy these properties.

        We will first prove existence. Let $\phi$ be as defined above.
        \begin{itemize}
            \item \emph{Missingness:} Suppose that $i\in[d]\setminus\mathcal A(x')$. Then there is an $S\subseteq\mathcal A(x')$ such that $i\in S$, so the sum is empty. This means that $\phi_i(f, x)=0$.
            \item \emph{Local Accuracy:} We want to show that $$\sum_{i=1}^d\phi_i(f, x) = f_x([d]) - f_x(\emptyset).$$ We can rewrite the left term as $$\sum_{i=1}^{d}\phi_i(f, x) = \sum_{S\subseteq\mathcal A(x')}\Gamma(S)f_x(S),$$ for certain coefficients $\Gamma(S)$. We will now determine these coefficients. Let $S\subseteq\mathcal A(x')$. Suppose that $0<|S|<|\mathcal A(x')|$. We find that \begin{align*}\Gamma(S) &= \sum_{i\in S}\frac{(|S| - 1)!(|\mathcal A(x')| - |S|)!}{|\mathcal A(x')|!} - \sum_{i\in\mathcal A(x')\setminus S}\frac{|S|!(|\mathcal A(x')| - |S|- 1)!}{|\mathcal A(x')|!}\\
            &= |S|\frac{(|S| - 1)!(|\mathcal A(x')| - |S|)!}{|\mathcal A(x')|!} - (|\mathcal A(x')| - |S|)\frac{|S|!(|\mathcal A(x')| - |S|- 1)!}{|\mathcal A(x')|!}\\
            &=0.\end{align*}
            Now suppose that $|S|=0$. We then get that $$\Gamma(S) = -\sum_{i\in\mathcal A(x')}\frac{(|\mathcal A(x')|-1)!}{|\mathcal A(x')|!} = -|\mathcal A(x')|\frac{(|\mathcal A(x')| - 1)!}{|\mathcal A(x')|!} = -1.$$
            Now suppose that $|S|=|\mathcal A(x')|!$. We now get that 
            $$\Gamma(S) = \sum_{i\in\mathcal A(x')}\frac{(|\mathcal A(x')|-1)!}{|\mathcal A(x')|!} = 1,$$ because of the same argument as when $|S|=0$. Using this and the fact that $\phi_i(f, x)=0$ for $i\in[d]\setminus\mathcal A(x')$ gives us that $$\sum_{i=1}^d\phi_i(f, x) = f_x([d]) - f_x(\emptyset)$$ as required.
            \item \emph{Restricted symmetry:} Let $i, j\in\mathcal A(x')$. Suppose that $$f_x(S\cup\set i)=f_x(S\cup\set j)\quad\text{for all $S\subseteq \mathcal A(x')\setminus\set{i, j}$.}$$ From the definition of $\phi$, we get that \begin{align*}
                \phi_i(f, x) &= \sum_{\substack{S\subseteq\mathcal A(x')\\i\in S}}\frac{(|S| - 1)!(|\mathcal A(x')| - |S|)!}{|\mathcal A(x')|!}[f_x(S) - f_x(S\setminus\set i)]\\
                &= \sum_{\substack{S\subseteq\mathcal A(x')\\j\in S}}\frac{(|S| - 1)!(|\mathcal A(x')| - |S|)!}{|\mathcal A(x')|!}[f_x(S) - f_x(S\setminus\set j)]\\
                &= \phi_j(f, x).
            \end{align*}
            This means that $\phi$ satisfies restricted symmetry.
            \item \emph{Restricted consistency:} Let $f':\X\to\bb R$ be a model and assume that $$f'_x(S) - f'_x(S\setminus\{i\})\geq f_x(S) - f_x(S\setminus\{i\})\quad\text{for all } S\subseteq \mathcal A(x').$$ Since all of the coefficients in the definition of $\phi$ are positive, we immediately see that $\phi_i(f', x)\geq\phi_i(f, x)$.
        \end{itemize}
        With this, we have proven that the given explanation satisfies missingness, local accuracy, restricted symmetry and restricted consistency.
        
        We will now prove that this explanation is in fact unique. Let $\phi$ be an explanation that satisfies local accuracy, missingness, restricted symmetry and restricted consistency. \autoref{lemma:res_consistency->constant_on_inducing} says that $\phi$ is constant on inducing. Because $\phi$ is constant on inducing, $\phi$ satisfies missingness and $\phi$ is locally accurate, we can use \autoref{lemm:loc<->eff} to conclude that $\psi^x$ is efficient for all $x\in\X$. Because $\phi$ is constant on inducing and $\phi$ satisfies restricted consistency, we can use \autoref{lemma:restricted_consistency->strong_monotonicity} to conclude that $\psi^x$ satisfies strong monotonicity for all $x\in\X$. Because $\phi$ is constant on inducing and satisfies both restricted consistency and restricted symmetry, we can use \autoref{lemma:restricted_symmetry->symmetry} to conclude that $\psi^x$ satisfies symmetry for all $x\in\X$. Let $x\in\X$. Because $\psi^x$ satisfies efficiency, strong monotonicity and symmetry, we can use \autoref{thm:shapley} to conclude that for $i\in \mathcal A(x')$ we have $$\psi^x_i(\nu) = \sum_{S\subseteq \mathcal A(x'):i\in S}\frac{(|S| - 1)!|\mathcal A(x')\setminus\set S|!}{|\mathcal A(x')|!}\nu^i(S).$$ Because $\phi$ is constant on inducing, we can say that for $i\in \mathcal A(x')$, \begin{align*}
        \phi_i(f, x) = \psi_i^x(\nu_{f_x}) &= \sum_{\overset{S\subseteq \mathcal A(x')}{i\in S}}\frac{(|S| - 1)!|\mathcal A(x')\setminus\set S|!}{|\mathcal A(x')|!}\nu_{f_x}^i(S)\\
        &= \sum_{\overset{S\subseteq \mathcal A(x')}{i\in S}}\frac{(|S| - 1)!|\mathcal A(x')\setminus\set S|!}{|\mathcal A(x')|!}\bric{\nu_{f_x}(S) - \nu_{f_x}(S\setminus\set i)}\\
        &= \sum_{\overset{S\subseteq \mathcal A(x')}{i\in S}}\frac{(|S| - 1)!(|\mathcal A(x')| - |S|)!}{|\mathcal A(x')|!}\bric{f_x(S) - f_x(S\setminus\set i)}.
    \end{align*}
\end{proof}

\chapter{SHAP as the solution to a regression problem}
\label{chapter:4}
A flaw of the SHAP-explanation is its computational efficiency. To calculate the Shapley values, it takes at least $2^d$ operations. This gives an exponential complexity, which is in a lot of cases too slow to use. The way to circumvent this is to use an approximation of the Shapley values. One method to approximate the Shapley values is the KernelSHAP-method. This is an approximation method that views the Shapley values as the solution to a linear regression problem and uses an algorithm to calculate this efficiently \cite{NIPS2017_8a20a862}.

\section{Reduction to a minimization problem}
We will first introduce more notation. Let $v = \begin{pmatrix}
    v_1&\cdots& v_d
\end{pmatrix}^T\in\bb R^d$ and let $\set{s_1, \dots, s_n}\subseteq[d]$. We will denote $v_{\set{s_1, \dots, s_n}}\in\bb R^n$ as the vector $\begin{pmatrix}
    v_{s_1} & \cdots & v_{s_n}
\end{pmatrix}^T$. Some examples illustrating this definition are: $$\begin{pmatrix}
    1\\3\\2\\4
\end{pmatrix}_{\set{1, 2}} =\begin{pmatrix}
    1\\3
\end{pmatrix},\qquad\begin{pmatrix}
    3\\4\\10
\end{pmatrix}_{\set{3}} = 10,\qquad\begin{pmatrix}
    4 \\2\\1\\3
\end{pmatrix}_{\set{2, 3, 4}} = \begin{pmatrix}
    2\\1\\3
\end{pmatrix}.$$ We note that for all $v\in\bb R^d$ and $i\in[d]$ that $v_{\set{i}} = v_i$. 

We will now first prove a lemma. The following lemma is very important for this section, because it allows us to assume that without loss of generality $\mathcal A(x')=[d]$ when working with the Shapley values. This lemma states that for all simplification functions $h_x:\set{0, 1}^d\to\X$, that we can create a new simplification function $h'_x:\set{0, 1}^{|x'|}\to\X$ where the simplified input with respect to $h'_x$ is the all-one vector. The theorem also says that for a specific choice of $h'_x$, that it does not matter if we pick $h_x$ or $h'_x$ when determining the Shapley values of a model $f$.
\begin{lemma}
\label{lemma:X_all_ones}
    Let $f: \X\to\bb R$ be a model. Let $x\in\X$ and suppose that $h_x:\set{0, 1}^d\to\X$ is a simplification function. Let $x'$ be the simplified input of $x$ with respect to $h_x$ and suppose that $x'$ is not the all-one vector. 
    There exists a simplification function $h_x': \set{0, 1}^{|x'|}\to\X$ such that the simplified input with respect to $h'_x$ is the all-one vector and $\shap(f_x)_{\mathcal A(x')} = \shap(f'_x)$ where $f'_x$ is the simplified model of $f$ with respect to $h'_x$.
\end{lemma}
\begin{proof}
    Define $n:=|x'|$. Let $p_1,\dots, p_n$ be an ordering of $\mathcal A(x')$. This means that $\mathcal A(x')=\set{p_1, \dots, p_n}$ and $p_i < p_j$ if $i<j$. We will now define a function $k:\set{0, 1}^n\to\set{0, 1}^d$. For $z\in\set{0, 1}^n$ and $i\in[d]$, we will define $$k(z)_i:=\begin{cases}
        z_j&\quad\text{if $i=p_j$ for some $j\in[n]$}\\
        0&\quad\text{otherwise.}
    \end{cases}$$
    Intuitively, can see this function $k$ as projecting a binary vector of length $n$ onto the indices in $\mathcal A(x')$ of a binary vector of length $d$.  

    From this definition, we can see that if $\1\in\set{0, 1}^n$ is the all-one vector, then $k(\1)=x'$. We can now define $h'_x:= h_x\circ k$. Since both $k$ and $h_x$ are injective, we can conclude that $h_x'$ is a simplification function with simplified input $x''$ that is the all-one vector.

    The fact that $\shap(f_x)_{\mathcal A(x')} = \shap(f'_x)$ follows directly from the definition of the Shapley values.
\end{proof}

As stated before, we would like to look at a more efficient way to approximate the Shapley values. To do this, we want to look at the Shapley values as the solution to a minimization problem. 
\begin{restatable}{theorem}{theoremtwo}
    \label{thm:2}
    Let $d\in\bb N$, let $f_x:\mathcal P([d])\to\bb R$ be a simplified model and let $\phi$ be an explanation. Suppose that $\phi(f, x)_{[d]\setminus \mathcal A(x')} =\boldsymbol{0}$ is the all-zero vector suppose that $(f_x(\emptyset), \phi(f, x))$ is a solution to $$(\hat\theta_0, \hat\theta) = \argmin_{(\theta_0, \theta)\in\bb R\times\bb R^d}\sum_{S\subseteq \mathcal A(x')}\bric{f_x(S) - \theta_0 - (1_S)^T\theta}^2\pi(|S|),$$ where $\pi(s) = \frac{|\mathcal A(x')|-1}{\binom{|\mathcal A(x')|}{s}s(|\mathcal A(x')|-s)}$. Then $\phi(f, x) = \shap(f_x)$.
\end{restatable}

The proof of this theorem depends on a number of technical lemmas, whose proofs we refer to \autoref{section:intermediate_results}.

\begin{proof}
    Firstly, because of \autoref{lemma:X_all_ones}, we can assume that, without loss of generality, $\mathcal A(x')=[d]$. 
    
    Since $\pi(0) = \pi(d) = \infty$, we must have that $$f_x(\emptyset) - \theta_0 = 0,\qquad f_x([d]) - \theta_0 - \sum_{i=1}^d\theta_i = 0.$$ From this first equation, we get that $\theta_0 = f_x(\emptyset)$. From the second equation, we get that $$\theta_d = f_x([d]) - f_x(\emptyset)-\sum_{i=1}^{d-1}\theta_i.$$ Now define $$A:= \begin{bmatrix}
    1&0&\cdots & 0\\
    0 & 1 &  \cdots & 0\\
    \vdots & \vdots & \ddots& \vdots\\
    0 & 0 & \cdots & 1\\
    -1 & -1 & \cdots & -1
\end{bmatrix},\qquad b :=\begin{bmatrix}
    0\\0\\\vdots\\0\\f_x([d]) - f_x(\emptyset)
\end{bmatrix}.$$ 
We can see that the mapping $$\zeta: \bb R^{d-1}\to\bb R\times\bb R^d,\quad \gamma\mapsto (f_x(\emptyset), A\gamma + b)$$ defines a bijection $$\bb R^{d-1} \longleftrightarrow \set{(\theta_0, \theta)\in\bb R\times\bb R^d: \theta_d =f_x([d]) - f_x(\emptyset)-\sum_{i=1}^{d-1}\theta_i,\quad\theta_0=f_x(\emptyset)}.$$ From this we can conclude that, because of the restrictions put on $(\theta_0, \theta)$ that \begin{align*}&\quad\argmin_{(\theta_0, \theta)\in\bb R\times\bb R^d}\sum_{S\subseteq [d]}\bric{f_x(S) - \theta_0 - (1_S)^T\theta}^2\pi(|S|) \\= &\quad A\brac{\argmin_{\gamma\in\bb R^{d-1}}\sum_{\substack{S\subseteq [d]\\S\neq [d], \emptyset}}\bric{f_x(S) - f_x(\emptyset) - (1_S)^T(A\gamma + b)}^2\pi(|S|)} + b.\end{align*}

Since $\mathcal P([d])$ is finite with $2^d$ elements, there exists a bijection $$\kappa: [2^d-2]\to\mathcal P([d])\setminus\set{\emptyset, [d]}.$$ We will now define $X\in\bb R^{(2^d-2)\times d}, W\in\bb R^{(2^{d-2})\times(2^d-2)}, y\in\bb R^{2^d-2}$ as the matrices \begin{align*}
    X = \begin{bmatrix}
    (1_{\kappa(1)})^T\\(1_{\kappa(2)})^T\\
    \vdots\\(1_{\kappa(2^{d}-2)})^T
\end{bmatrix},\qquad W &= \begin{bmatrix}
    \pi(|\kappa(1)|) & 0 &0& \cdots & 0\\
    0 & \pi(|\kappa(2)|) & 0 & \cdots & 0\\
    0 & 0 & \ddots & & \vdots\\
    \vdots & \vdots & & \ddots & \vdots\\
    0 & 0 & \cdots & \cdots & \pi(|\kappa(2^d-2)|)
\end{bmatrix}, \\y &= \begin{bmatrix}
    f_x(\kappa(1)) - f_x(\emptyset)\\
    f_x(\kappa(2)) - f_x(\emptyset)\\
    \vdots\\
    f_x(\kappa(2^d - 2)) - f_x(\emptyset)
\end{bmatrix}.\end{align*}This means that $X$ is the matrix with as its rows, all the vectors in $\set{0, 1}^d$ without the all-zero and the all-one vectors. We also have that $W$ is the diagonal matrix such that $W_{kk}$ corresponds with the weight of row $k$ of $X$. Finally, $y$ is the vector such that $y_k$ corresponds with evaluating $f_x$ on the set corresponding to row $k$ of $X$. 

Using these matrices, we can rewrite our minimization problem to \begin{align*}&\quad\ \argmin_{\gamma\in\bb R^{d-1}}\sum_{\substack{S\subseteq [d]\\S\neq [d], \emptyset}}\bric{f_x(S) - f_x(\emptyset) - (1_S)^T(A\gamma + b)}^2\pi(|S|) \\&= \argmin_{\gamma\in\bb R^{d-1}}(y - X(A\gamma + b))^TW(y - X(A\gamma + b))=:\hat\gamma.\end{align*} To determine this minimum, we take the derivative with respect to $\gamma$. From equation (84) from Petersen and Pedersen \cite{matrixcookbook}, we get that
$$\frac{\partial}{\partial\gamma}(y-X(A\gamma + b))^TW(y - X(A\gamma + b)) = -2(XA)^TW(y - Xb - XA\gamma).$$
Under the assumption that $(XA)^TWXA$ is invertible (see \autoref{corollary:inverse_product}), we can find extrema by solving for 0:\begin{align*}
    0&=-2(XA)^TW(y - Xb - XA\gamma)\\
    \gamma &= [(XA)^TWXA]^{-1}(XA)^TW(y - Xb).
\end{align*}
We have now found an extremum, but we must still show that this is in fact a minimum. To do this, we will define the function $g:\bb R^{d-1}\to\bb R$ by $$g(\gamma) = (y-X(A\gamma +b))^TW(y - X(A\gamma + b))$$ and we will show that this function is convex.

We will first show that the function $\alpha(\theta) = \theta^TW\theta$, with $W$ as defined above, is convex. Equation (98) from \cite{matrixcookbook} says that $$\frac{\partial^2 g}{\partial \theta \ \partial\theta^T} = W + W^T = 2W.$$ Because $W$ is a diagonal matrix with nonnegative indices, we can conclude that the Hessian matrix of $\alpha$ is positive semi-define. We conlcude that $\alpha$ is convex by \autoref{theorem:convex_second_derivative}.

Now define the function $\beta: \bb R^{d-1}\to\bb R^{2^d-2}$ as $$\beta(\gamma) = -XA\gamma + y - Xb.$$ Since $\beta$ is an affine function, we can use \autoref{lemma:convex+affine} to conclude that the function $\alpha\circ \beta$ is convex. Since $g = \alpha\circ \beta$, we can conclude that $g$ is convex.

Using \autoref{corollary:convex_minimum}, we conclude that the extremum that we found is in fact a minimum. 

Using \autoref{lemma:left_side_product}, \autoref{lemma:right_side_product} and the definition of the matrix-product, we find that for $i\in\set{1, \dots, d-1}$ \begin{align*}
    \hat\gamma_i &= \sum_{j=1}^{2^d - 2}\bric{[(XA)^TWXA]^{-1}(XA)^T}_{ij}[W(y - Xb)]_j\\
    &= \sum_{\substack{j=1\\\kappa(j)_d=1}}^{2^d-2}\brac{\frac{d}{d-1}(\kappa(j)_i - 1) + \frac{1}{d-1}(d - |\kappa(j)|)}\pi(|\kappa(j)|)(f_x(\kappa(j) - f_x([d])) \\
    &\quad\quad+ \sum_{\substack{j=1\\\kappa(j)_d=0}}^{2^d-2}\brac{\frac{d}{d-1}X_{ji} - \frac{1}{(d-1)}|\kappa(j)|}\pi(|\kappa(j)|)(f_x(\kappa(j)) - f_x(\emptyset)).
\end{align*}
We can now rewrite this sum to $$\hat\gamma_i = \sum_{S\subseteq[d]}\Sigma(S)f_x(S),$$
where $\Sigma(S)$ is the coefficient of $f_x(S)$ in the result of the above sum. We will now determine the values of $\Sigma(S)$ by case-distinction.

Let $S\subseteq[d]$ and let $\boldsymbol{0}, \1\in\set{0, 1}^{d-1}$ denote the all-zero and all-one vectors respectively.

\emph{Case $S = [d]$:} Using \autoref{proposition:5} , we get that\begin{align*}\Sigma(S) &= -\sum_{\substack{z\in\set{0, 1}^{d-1}\setminus\set\1\\z_d=1}}\brac{\frac{d}{d-1}(z_i-1) + \frac{1}{d-1}(d-|z|)}\pi(d-|z|)
= \frac{1}{d} = \frac{1}{|S|\binom{d}{|S|}}.\end{align*}

\emph{Case $S = \emptyset$:} Using \autoref{proposition:6} , we get that \begin{align*}\Sigma(S) = -\sum_{\substack{z\in\set{0, 1}^d\setminus\set{\boldsymbol{0}}\\z_d = 0}}\brac{\frac{d}{d-1}z_i - \frac{1}{d-1}|z|}\pi(|z|) = -\frac{1}{d} = -\frac{1}{(1 + |S|)\binom{d}{1 + |S|}}.\end{align*}

\emph{Case $S\neq[d], \emptyset$ and $i, d\in S$:} Using part (a) from \autoref{proposition:7}, we get that $$\Sigma(S) = \frac{1}{d-1}(d - |S|)\pi(d - |S|) = \frac{1}{|S|\binom{d}{|S|}}.$$

\emph{Case $S\neq[d], \emptyset$ and $i, d\notin S$:} Using part (b) from \autoref{proposition:7}, we get that $$\Sigma(S) = -\brac{\frac{1}{d-1}|S|}\pi(|S|) = -\frac{1}{(|S| + 1)\binom{d}{|S|+1}}.$$

\emph{Case $S\neq[d], \emptyset$, $i\in S$ and $d\notin S$:} Using part (c) from \autoref{proposition:7}, we get that $$\Sigma(S) = \brac{\frac{d}{d-1} - \frac{1}{d-1}|S|}\pi(|S|) = \frac{1}{|S|\binom{d}{|S|}}.$$

\emph{Case $S\neq[d], \emptyset$, $i\notin S$ and $d\in S$:} Using part (d) from \autoref{proposition:7}, we get that $$\Sigma(S) = \brac{-\frac{d}{d-1} + \frac{1}{d-1}(d-|S|)}\pi(d - |S|) = -\frac{1}{(|S| + 1)\binom{d}{|S|+1}}.$$

With these values, we can now calculate the entire sum. We get that\begin{align*}
    \hat\gamma_i &= \sum_{S\subseteq[d]}\Sigma(S)f_x(S)\\
    &= \sum_{\substack{S\subseteq[d]\\i\in S}}\brac{\Sigma(S)f_x(S) + \Sigma(S\setminus\set i)f_x(S\setminus\set i)}\\
    &= \sum_{\substack{S\subseteq[d]\\i\in S}}\brac{\frac{1}{|S|\binom{d}{|S|}}f_x(S) - \frac{1}{(|S\setminus\set i|+1)\binom{d}{|S\setminus\set i| + 1}}f_x(S\setminus\set i)}\\
    &= \sum_{\substack{S\subseteq[d]\\i\in S}}\frac{1}{|S|\binom{d}{|S|}}[f_x(S) - f_x(S\setminus\set i)].
\end{align*}
With this, we can see that if $\hat\theta = A\hat\gamma + b$, then for all $i\in\set{1, \dots, d-1}$ we have that $\hat\theta_i = \hat\gamma_i=\shap(f_x)_i$. We now only need to check for $i=d$ if this is also true. Because the Shapley values satisfy local accuracy, we get that $$\sum_{i=1}^d\shap(f_x)_i = f_x([d]) - f_x(\emptyset).$$ From this, we can also get that $$\shap(f_x)_d = f_x([d]) - f_x(\emptyset) - \sum_{i=1}^{d-1}\shap(f_x)_i.$$ Using the fact that $\shap(f_x)_i = \hat\theta_i$ for $i\in\set{1, \dots, d-1}$ gives us that $$\shap(f_x)_d= f_x([d]) - f_x(\emptyset) - \sum_{i=1}^{d-1}\hat\theta_i.$$ We can now observe that this is exactly the constraint that we put on $\hat\theta_d$, so we have that $\shap(f_x)_d=\hat\theta_d$. With this, we have that $(\hat\theta_0, \hat\theta) = (f_x(\emptyset), \shap(f_x))$ and since this is the only solution to the regression problem, we have proven \autoref{thm:2}.
\end{proof}

\section{Intermediate results}
\label{section:intermediate_results}
In this section, we calulate the intermediate results that are used in the proof of \autoref{thm:2}, which is given in the previous section. Let $X, W, A, b, y$ and $\pi$ be as stated in the previous section. For two matrices $M, N$, we will denote $M_{ij}$ as the index on row $i$ and column $j$ of $M$. For products of matrices, we will use square brackets for the same purpose: $[MN]_{ij}$ is the index on row $i$ and column $j$ of $MN$. 

\begin{proposition}
\label{proposition:1}
The indices of $XA$ can be determined as follows:
$$[XA]_{ij} = X_{ij} - X_{im}.$$
\end{proposition}
\begin{proof}
    From the definition of matrix multiplication, we get that for $i\in\\set{1, 2, \dots, n}$ and $j\in\set{1, 2, \dots, m-1}$ $$[XA]_{ij} = \sum_{k=1}^mX_{ik}A_{kj} = X_{ij} - X_{im}.$$
\end{proof}

\begin{proposition}
\label{proposition:2}
    Let $k\in\set{1, 2, \dots, d-1}$. We have that $\pi(k) = \pi(d-k)$.
\end{proposition}
\begin{proof}
    We have that $$\pi(k) = \frac{d-1}{\binom{d}{k}k(d-k)} = \frac{d-1}{\binom{d}{d-k}k(d-k)}= \pi(d-k).$$
\end{proof}

\begin{proposition}
\label{proposition:3}
    The following equalities hold \begin{align*}\text{(a)}\qquad\qquad\qquad&2\sum_{s=1}^{d-1}\pi(s)\binom{d-3}{s-1} &&= \frac{d-1}{d},\\\text{(b)}\qquad\qquad\qquad&2\sum_{s=1}^{d-1}\pi(s)\binom{d-3}{s-2} &&= \frac{d-1}{d}.\end{align*}
    \vspace{0pt}
\end{proposition}
\begin{proof}
\begin{itemize}
    \item[\emph{(a)}:] We have that \begin{align*}
            2\sum_{s=1}^{d-1}\pi(s)\binom{d-3}{s-1} &= 2\sum_{s=1}^{d-2}\pi(s)\binom{d-3}{s-1}\\
            &= 2\sum_{s=1}^{d-2}\frac{(d-1)s!(d-s)!(d-3)!}{s(d-s)d!(s-1)!(d-s-2)!}\\
            &= 2\sum_{s=1}^{d-2}\frac{(d-s-1)}{d(d-2)}\\
            &= 2\sum_{s=1}^{d-2}\frac{d-1}{d(d-2)} - \frac{2}{d(d-2)}\sum_{s=1}^{d-2}s\\
            &= 2\frac{d-1}{d} - \frac{2}{d(d-2)}\frac{(d-2)(d-1)}{2}\\
            &= 2\frac{d-1}{d} - \frac{d-1}{d}\\
            &= \frac{d-1}{d}.
        \end{align*}
    \item[\emph{(b)}:] We have that \begin{align*}
            2\sum_{s=1}^{d-1}\pi(s)\binom{d-3}{s-2} &= 2\sum_{s=2}^{d-1}\pi(s)\binom{d-3}{s-2}\\
            \intertext{Substituting $s$ by $d-s$ gives us:}
            &= 2\sum_{s=1}^{d-2}\pi(d-s)\binom{d-3}{d-s+2}\\
            &= 2\sum_{s=1}^{d-2}\pi(s)\binom{d-3}{s-1}\\
            \intertext{Using par (a) gives us that }
            2\sum_{s=1}^{d-1}\pi(s)\binom{d-3}{s-2}&= \frac{d-1}{d}.
        \end{align*}
\end{itemize}
\end{proof}

\begin{lemma}
\label{lemma:XATWXA_solution}
    We have that $$(XA)^TWXA = \frac{d-1}dI + \frac{d-1}{d}J,$$ where $I, J\in\bb R^{(d-1)\times(d-1)}$ with $I$ the identity and $J$ the matrix of all ones.
\end{lemma}
\begin{proof}
Let $i\in \set{1, \dots, 2^{d} - 2}$, $j\in\set{1, \dots, d-1}$ and denote $\boldsymbol{0}\in\set{0, 1}^{d-1}$ as the all-zero vector. From the definition of $A$, we get that $[XA]_{ij} = X_{ij} - X_{id}$. 
Because the rows of $X$ are all of the binary vectors excluding the all-zero and all-one vectors, we get that the rows of $XA$ are given by all vectors in $\set{0, 1}^{d-1}\setminus\set{\boldsymbol{0}}$ and all vectors in $\set{0, -1}^{d-1}\setminus\set{\boldsymbol{0}}$.

Now let $k\in\set{1, \dots, 2^d-2}$. If $[\kappa(k)]_d = 0$, then row $k$ of $XA$ is a vector in $\set{0, 1}^{d-1}$ and if $[\kappa(k)]_d=1$, then row $k$ of $XA$ is a vector in $\set{0, -1}^{d-1}$. Using \autoref{lemma:outer_product_sum}, we get that 

$$(XA)^TWXA = \sum_{z\in\set{0, 1}^{d-1}\setminus\set{\boldsymbol{0}}}\pi(|z|)zz^T + \sum_{z\in\set{0, -1}^{d-1}\setminus\set{\boldsymbol{0}}}\pi(d - |z|)zz^T = 2\sum_{z\in\set{0, 1}^{d-1}\setminus\set{\boldsymbol{0}}}\pi(|z|)zz^T.$$

From \autoref{lemma:sum_binary_vectors}, we get that \begin{align*}2\sum_{z\in\set{0, 1}^{d-1}\setminus\set{\boldsymbol{0}}}\pi(|z|)zz^T
&=2\sum_{s=1}^{d-1}\pi(s)\sum_{\substack{z\in\set{0, 1}^{d-1}\\|z| = s}} zz^T\\
&= 2\sum_{s=1}^{d-1}\pi(s)\brac{\binom{d-3}{s-1}I + \binom{d-3}{s-2}J}\\
&= \brac{2\sum_{s=1}^{d-2}\pi(s)\binom{d-3}{s-1}}I + \brac{2\sum_{s=2}^{d-1}\pi(s)\binom{d-3}{s-2}}J\\
&= \frac{d-1}{d} I + \frac{d-1}{d}J.
\end{align*}The final step follows from \autoref{proposition:3}
\end{proof}

\begin{corollary}
    \label{corollary:inverse_product}
    We have that $$[(XA)^TWXA]^{-1} = \frac{d}{d-1}I - \frac{1}{d-1}J.$$
\end{corollary}
\begin{proof}
    Let $\1\in\bb R^{d-1}$ be the vector with only ones. We now have that $J = \1\1^T$. We also have that $\brac{\frac{d-1}{d}I}^{-1} = \frac{d}{d-1}I$. From the Sherman-Morrison-Woodbury formula (\autoref{lemma:sherman-morrison-woodbury}), we get that \begin{align*}
        [(XA)^TWXA]^{-1} &= \frac{d}{d-1}I - \brac{\frac{d}{d-1}}^2\frac{d-1}{d}\frac{1}{1 + \frac{d}{d-1}\frac{d-1}{d}(d-1)}\1\1^T\\
        &= \frac{d}{d-1}I - \frac{1}{d-1}J.
    \end{align*}
\end{proof}

\begin{lemma}
\label{lemma:left_side_product}
    Let $i\in\set{1, \dots, d-1}, j\in\set{1, \dots, 2^d-2}$. We have that $$\bric{[(XA)^TWXA]^{-1}(XA)^T}_{ij}  = \begin{cases}
         \frac{d}{d-1}(X_{ji} - 1) + \frac{1}{(d-1)}(d - |\kappa(j)|), \quad&\text{if $X_{jd} = 1$}\\
         \frac{d}{d-1}X_{ji} - \frac{1}{(d-1)}|\kappa(j)|, \quad&\text{if $X_{jd} = 0$.}
     \end{cases}$$
\end{lemma}
\begin{proof}
    First, from \autoref{corollary:inverse_product}, we get that $$[(XA)^TWXA]^{-1}(XA)^T = \frac{d}{d-1}(XA)^T - \frac{1}{d-1}J(XA)^T.$$Now suppose that $X_{jd} = 1$. We get that for $k\in\set{1, \dots, d-1}$ that $[XA]_{jk} = X_{jk} - 1$. We now get that \begin{align*}
        \bric{[(XA)^TWXA]^{-1}(XA)^T}_{ij} &= \frac{d}{d-1}[(XA)^T]_{ij} - \frac{1}{d-1}[J(XA)^T]_{ij}\\
        &= \frac{d}{d-1}[XA]_{ji} - \frac{1}{d-1}\sum_{k=1}^{d-1}[XA]_{jk}\\
        &= \frac{d}{d-1}[XA]_{ji} - \frac{1}{d-1}\sum_{k=1}^{d-1}(X_{jk} - 1)\\
        &= \frac{d}{d-1}[XA]_{ji} + \frac{1}{d-1}(d - |\kappa(j)|)\\
        &= \frac{d}{d-1}(X_{ji} - 1) + \frac{1}{d-1}(d - |\kappa(j)|),
    \end{align*}
    where we use that the $j$'th row of $X$ is $\kappa(j)$.

    Now suppose that $X_{jd} = 0$. We get, for $k\in\set{1, \dots, d-1}$ that $[XA]_{jk} = X_{jk}$. This gives us that \begin{align*}
        \bric{[(XA)^TWXA]^{-1}(XA)^T}_{ij} &= \frac{d}{d-1}[(XA)^T]_{ij} - \frac{1}{d-1}[J(XA)^T]_{ij}\\
        &= \frac{d}{d-1}[XA]_{ji} - \frac{1}{d-1}\sum_{k=1}^{d-1}[XA]_{jk}\\
        &= \frac{d}{d-1}[XA]_{ji} - \frac{1}{d-1}\sum_{k=1}^{d-1}X_{jk}\\
        &= \frac{d}{d-1}[XA]_{ji} - \frac{1}{d-1}|\kappa(j)|.
    \end{align*}
\end{proof}

\begin{lemma}
\label{lemma:right_side_product}
    Let $i\in\set{1, \dots, 2^{d}-2}$. We have that $$[W(y - Xb)]_i = \begin{cases}
        \pi(|\kappa(i)|)(f_x(\kappa(i)) - f_x([d])),\quad&\text{if $X_{id}=1$}\\
        \pi(|\kappa(i)|)(f_x(\kappa(i))- f_x(\emptyset)),\quad &\text{if $X_{id} = 0$.}
    \end{cases}$$
\end{lemma}
\begin{proof}
    Suppose that $X_{id}=1$, then, using the definition of matrix multiplication, we get that $[Xb]_i = f_x([d]) - f_x(\emptyset)$. using this, we get that $$[y - Xb]_i = f_x(\kappa(i)) - f_x(\emptyset) + f_x(\emptyset) - f_x([d]) = f_x(\kappa(i)) - f_x([d]).$$ From the fact that $W$ is diagonal and $W_{ii} = \pi(|\kappa(i)|)$, we get that $$[W(y - Xb)]_i = \pi(|\kappa(i)|)(f_x(\kappa(i)) - f_x([d])).$$ Now suppose that $X_{id} = 0$. We now find that $[Xb]_i = 0$, so $[y - Xb]_i = \kappa(i) - f_x(\emptyset)$. We now find that $$[W(y - Xb)]_i = \pi(|\kappa(i)|)(f_x(\kappa(i)) - f_x(\emptyset)).$$
\end{proof}

\begin{proposition}
\label{proposition:5}
    The following equality holds $$\sum_{\substack{z\in\set{0, 1}^{d-1}\setminus\set\1\\z_d=1}}\brac{\frac{d}{d-1}(z_i-1) + \frac{1}{d-1}(d-|z|)}\pi(d-|z|) = -\frac{1}{d}.$$
    \vspace{0pt}
\end{proposition}
\begin{proof}
    We will split this sum into two parts. The first part is the following:
        \begin{align*}
            \sum_{\substack{z\in\set{0, 1}^d\setminus\set\1\\z_d=1}}\frac{d}{d-1}(z_i-1)\pi(d - |z|) &= \sum_{s=1}^{d-1}\pi(d-s)\sum_{\substack{z\in\set{0, 1}^d\setminus\set\1\\z_d=1\\|z|=s}}\frac{d}{d-1}(z_i - 1)\\
            &= -\sum_{s=1}^{d-1}\pi(d-s)\sum_{\substack{z\in\set{0, 1}^d\setminus\set\1\\z_i=0, z_d=1\\|z|=s}}\frac{d}{d-1}\\
            &= -\sum_{s=1}^{d-1}\pi(s)\frac{d}{d-1}\binom{d-2}{s-1}\\
            &= -\sum_{s=1}^{d-1}\frac{(d-1)s!(d-s)!d(d-2)!}{d!s(d-s)(d-1)(s-1)!(d-s-1)!}\\
            &= -\sum_{s=1}^{d-1}\frac{1}{(d-1)}\\
            &= -1.
        \end{align*}
        The second part of the sum is given by \begin{align*}
            \sum_{\substack{z\in\set{0, 1}^d\setminus\set\1\\z_d=1}}\frac{1}{d-1}(d-|z|)\pi(d-|z|) &= \sum_{s=1}^{d-1}\frac{d-s}{d-1}\pi(d-s)\sum_{\substack{z\in\set{0, 1}^d\\z_d=1\\|z|=s}}1\\
            &= \sum_{s=1}^{d-1}\frac{d-s}{d-1}\pi(d-s)\binom{d-1}{s-1}\\
            &= \sum_{s=1}^{d-1}\frac{d-s}{d-1}\pi(s)\binom{d-1}{s-1}\\
            &= \sum_{s=1}^{d-1}\frac{(d-s)(d-1)s!(d-s)!(d-1)!}{(d-1)d!s(d-s)(s-1)!(d-s)!}\\
            &=\sum_{s=1}^{d-1}\frac{1}{d}\\
            &= \frac{d-1}{d}.
        \end{align*}
        We can now conclude that $$\sum_{\substack{z\in\set{0, 1}^{d-1}\setminus\set\1\\z_d=1}}\brac{\frac{d}{d-1}(z_i-1) + \frac{1}{d-1}(d-|z|)}\pi(d-|z|) = -1 + \frac{d-1}{d} = -\frac1d.$$
\end{proof}

\begin{proposition}
\label{proposition:6}
    The following equality holds: $$\sum_{\substack{z\in\set{0, 1}^d\setminus\set{\boldsymbol{0}}\\z_d=0}}\brac{\frac{d}{d-1}z_i - \frac{1}{d-1}|z|}\pi(|z|) = \frac{1}{d}.$$
\end{proposition}
\begin{proof}
    We will split this sum into two parts. The first part reduces as follows:
        \begin{align*}
            \sum_{\substack{z\in\set{0, 1}^d\setminus\set\0\\z_d=0}}\frac{d}{d-1}z_i\pi(|z|) &= \sum_{s=1}^{d-1}\frac{d}{d-1}\pi(s)\sum_{\substack{z\in\set{0, 1}^d\\z_d=0\\|z|=s}}z_i\\
            &= \sum_{s=1}^{d-1}\frac{d}{d-1}\pi(s)\binom{d-2}{s-1}\\
            &= \sum_{s=1}^{d-1}\frac{d(d-1)s!(d-s)!(d-2)!}{(d-1)d!s(d-s)(s-1)!(d-s-1)!}\\
            &= \sum_{s=1}^{d-1}\frac{1}{(d-1)}\\
            &= 1.
        \end{align*}
        We can now reduce the second part as follows:
        \begin{align*}
            \sum_{\substack{z\in\set{0, 1}^d\setminus\set\0\\z_d=0}}\frac{1}{d-1}|z|\pi(|z|) &= \sum_{s=1}^{d-1}\frac{1}{d-1}s\pi(s)\sum_{\substack{z\in\set{0, 1}^{d}\setminus\set\0\\z_d=0\\|z|=s}}1\\
            &= \sum_{s=1}^{d-1}\frac{s}{d-1}\pi(s)\binom{d-1}{s}\\
            &= \sum_{s=1}^{d-1}\frac{s(d-1)s!(d-s)!(d-1)!}{(d-1)d!s(d-s)s!(d-s-1)!}\\
            &= \sum_{s=1}^{d-1}\frac{1}{d}\\
            &= \frac{d-1}{d}.
        \end{align*}
        We can now conclude that $$\sum_{\substack{z\in\set{0, 1}^d\setminus\set{\boldsymbol{0}}\\z_d=0}}\brac{\frac{d}{d-1}z_i - \frac{1}{d-1}|z|}\pi(|z|) = 1 - \frac{d-1}{d} = \frac1d.$$
\end{proof}

\begin{proposition}
\label{proposition:7}
    The following identities hold:
    \begin{align*}
        &\text{(a)}\qquad\qquad\qquad\qquad \frac{1}{d-1}(d-s)\pi(d-s) &&= \frac{1}{s\binom{d}{s}}\\
        &\text{(b)}\qquad\qquad\qquad\qquad\brac{\frac{1}{d-1}s}\pi(s) &&= \frac{1}{(s+1)\binom{d}{s+1}}\\
        &\text{(c)}\qquad\qquad\qquad\qquad\brac{\frac{d}{d-1} - \frac{1}{d-1}s}\pi(s)&&= \frac{1}{s\binom{d}{s}}\\
        &\text{(d)}\qquad\qquad\qquad\qquad\brac{\frac{d}{d-1} - \frac{1}{d-1}(d-s)}\pi(d - s) &&= \frac{1}{(d-s)\binom{d}{s}}\qquad\qquad\qquad
    \end{align*}
    \vspace{0pt}
\end{proposition}
\begin{proof}
\begin{itemize}
    \item[\emph{(a)}]: We have that \begin{align*}
            \frac{1}{d-1}(d-s)\pi(d-s) &= \frac{1}{d-1}(d-s)\pi(s)\\
            &= \frac{(d-s)(d-1)s!(d-s)!}{(d-1)d!s(d-s)}\\
            &= \frac{s!(d-s)!}{d!s}\\
            &= \frac{1}{s\binom{d}{s}}.
        \end{align*}

    \item[\emph{(b)}]: We have that \begin{align*}
            \brac{\frac{1}{d-1}s}\pi(s) &= \frac{s(d-1)s!(d-s)!}{(d-1)d!s(d-s)}\\
            &= \frac{s!(d-s-1)!}{d!}\\
            &= \frac{(s+1)!(d-s-1)!}{(s+1)d!}\\
            &= \frac{1}{(s+1)\binom{d}{s+1}}.\end{align*}

    \item[\emph{(c)}]: We have that \begin{align*}
            \brac{\frac{d}{d-1} - \frac{1}{d-1}s}\pi(s) &= \frac{d-s}{d-1}\pi(s)\\
            &= \frac{1}{s\binom{d}{s}},
        \end{align*}
        because of (a).

    \item[\emph{(d)}]: We have that \begin{align*}
            \brac{\frac{d}{d-1} - \frac{1}{d-1}(d-s)}\pi(d - s) &= \frac{s}{d-1}\pi(s)\\
            &= \frac{1}{(s+1)\binom{d}{s+1}},\end{align*}
            because of (b).
\end{itemize}
\end{proof}

\section{Discussion on Lundberg and Lee}
Lundberg and Lee also give a proof for \autoref{thm:2}. Their proof misses one detail. The strategy that Lundberg and Lee use to prove \autoref{thm:2} is very similar to the one in this thesis. The major difference is that they parametrize the weight function $\pi$ with a parameter $c$. They do this by defining $$\pi_c(s) = \begin{cases}
    \frac{d-1}{\binom ds(d-s)s},&\quad\text{if $s\neq 0, d$}\\
    c,&\quad\text{if $s=0, d$}
\end{cases}$$ They then let $c\to\infty$ and show that this limit gives the Shapley values. The problem with this approach is that Lundberg and Lee make the assumption that \footnotesize{ $$\lim_{c\to\infty}\argmin_{(\theta_0, \theta)\in\bb R\times\bb R^d}\sum_{S\subseteq [d]}\bric{f_x(S) - \theta_0 - (1_S)^T\theta}^2\pi_c(|S|) = \argmin_{(\theta_0, \theta)\in\bb R\times\bb R^d}\lim_{c\to\infty}\sum_{S\subseteq [d]}\bric{f_x(S) - \theta_0 - (1_S)^T\theta}^2\pi_c(|S|).$$}
\normalsize In other words, Lundberg and Lee assume that the $\lim_{c\to\infty}$ and the  $\argmin_{(\theta_0, \theta)\in\bb R\times\bb R^d}$ operators commute. This is not always the case, so this needs to be proven separately. It is possible that this can be proven with the use of Berge's maximum theorem \cite[116]{Bonsall_1963}.

\clearpage

\chapter{Conclusion}
\section{Summary}
In this thesis, we proved that the Shapley values are the unique explanation that satisfies local accuracy, missingness, restricted symmetry and restricted consistency. We did this by making use of a theorem from Young \cite{young1985monotonic} and by setting up a one-to-one correspondence between machine learning models and cooperative games. 

We also proved another theorem that reduces the Shapley values to the solution of regression problem. This was done by eliminating degrees of freedom and then solving a matrix product.

Lundberg and Lee \cite{NIPS2012_c399862d} gave a proof for both of these theorems as well. In this thesis, we discussed their proofs and formulations and gave corrections where necessary.

\addcontentsline{toc}{chapter}{Bibliografie}
\printbibliography

\clearpage

\chapter*{Popular Summary}
Imagine the following. A couple of years ago, you graduated from college and in the past few years, you worked really hard. In this period, you managed to save up a lot of money to (hopefully) be able to take on a mortgage and by a house for you and your partner. You go to the bank and after a short interview, the banker asks you for some of your data. The bank has just started using a new method to determine wether its clients are eligible for a loan. This new method makes use of a neural network, which is a form of artificial intelligence, which can say ``yes'' or ``no'' to a loan request. After hearing this, you fill out a form and hand it to the banker. He goes away and when he comes back, he has bad news: you didn't get the loan. Baffled, you look to him and ask him: ``why did I not get the loan?'', to which the banker replies: ``I don't know, this is just what our model determined.''

In this scenario, there is a clear flaw: we cannot interpret the decision made by a neural network. This is a flaw that is shared by most Machine Learning models. A solution to this problem is the use of an explanation. An explanation of a machine learning model is an algorithm that assigns every variable a number to indicate its importance. A higher value would mean that the variable was very important for the neural network to make its decision.

In the previous example, the model makes its decision as follows. Given a lot of data, it calculates a value $y$. If $y\geq 0$, the model says that the client is applicable for a loan. If $y<0$, then the request for a loan gets denied. Suppose that our client paid all of his previous loans and interest on time. This would make him more eligible for a loan, so an explanation would give this a positive score. The higher this score is, the more important the model thinks that it is to pay off your previous loans on time. Now suppose that the client still has some open loans. This would make him less eligible for a loan, so an explanation would give his current loans a negative value. 

As mentioned, an explanation gives all of the parameters in the input of a machine learning model a value based on how important they are. The way to choose these values depend on what method is used. This thesis focusses on the SHAP-explanation. This is an explanation that was introduced in 2017 by Lundberg and Lee. The value that this explanation gives to each variable is called the Shapley value.

We can see what these Shapley values look like with the use of an image. In \autoref{fig:popular_explanation}, we have calculated the SHAP-explanation. On the left, we see the original image. For us humans, this is clearly an acoustic guitar. Next to the original image, we have three images. Our model thinks that it is most likely that the image is an acoustic guitar, after this an electric guitar and after this a banjo. The highlighted pixels are the most important pixels in making this decision. We can see that the top part of the guitar and a bit of the neck of the guitar were very important in determining the fact that this is an acoustic guitar. We can also see that there is not a lot of reason to think that this is an electric guitar or a banjo, because a lot of the pixels are black. The actual model actually says that the chance of the image being an acoustic guitar is about 25 times as high as the chance that the image is a banjo or and electric guitar, so this corresponds with most of the pixels being black in the figure.

There is unfortunately one big problem with the SHAP-explanation. The calculations are very slow. These calculations are so slow that even with the use of the fastest programming language, it will take at least 300 years to calculate one Shapley value in \autoref{fig:popular_explanation}. This value is very theoretical and assumes very optimal code and a very fast model. In most practical cases, it will take a lot longer than 300 years. Since this takes so long, we want to be able to approximate the Shapley values. In this thesis, I proved a theorem that makes such an approximation possible. With the use of this approximation, it took approximately 10 minutes to calculate \autoref{fig:popular_explanation}, which is a lot better than a minimum of 300 years.

\begin{figure}
    \centering
    \includegraphics[width=\linewidth]{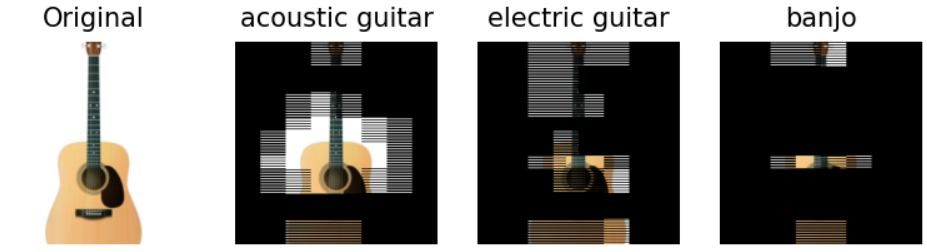}
    \caption{Visualisation of an explanation on an acoustic guitar. The model used to create this explanation is AlexNet \cite{NIPS2012_c399862d}.}
    \label{fig:popular_explanation}
\end{figure}

\addcontentsline{toc}{chapter}{Populaire samenvatting}

\appendix
\chapter{Symmetry in allocation procedures}
\label{appendix:symmetry}
Symmetry is a definition that is used in a lot of literature about game theory. Unfortunately, there are two definitions that are used in different literature that are both called symmetry. The first definition is the definition given by Young \cite{young1985monotonic} and the one defined in this thesis. In this chapter, if $\sigma:[n]\to[n]$ is a permutation and $\nu:\mathcal P([n])\to\bb R$, then we will denote $\sigma(\nu)$ as the cooperative game such that $\sigma(\nu)(S) = \nu(\sigma(S))$ with $\sigma(S) = \set{\sigma(s): s\in S}$ for all $S\subseteq[n]$.
\begin{property}[Symmetry]
    We say that $\phi$ is \textit{symmetric} if for all permutations $\sigma:[n]\to[n]$ we have \begin{equation}\phi_{\sigma (i)}(\sigma(\nu)) = \phi_i(v).\label{eq:symmetry}\end{equation}
    \label{prop:symmetry}
\end{property}

Another definition that is often used is a definition that is very similar to the definition of restricted symmetry for explanations (\autoref{prop:res_symmetry}).

\begin{property}[New Symmetry]
    Let $\psi$ be an allocation procedure with players $\mathcal A(x')$. We say that $\psi$ satisfies \textit{new symmetry} if the following implication holds. Let $i, j\in \mathcal A(x')$. If $$\nu(S\cup\set i) = \nu(S\cup\set j)\quad\text{for all $S\subseteq \mathcal A(x')\setminus\set{i, j}$},$$ then $\psi_i(\nu) = \psi_j(\nu)$.
\end{property}
New Symmetry is also a definition that is often used in literature \cite{winter2002shapley}. This chapter will prove that given certain conditions, these two definitions are equivalent.

Before we prove this equivalence, we need to look at a property of permutations.
\begin{lemma}
    \label{lemma:perm_prod}
    Let $D$ be a finite set. Every permutation $\sigma:D\to D$ can be written as the composition of functions $\sigma_{ab}:D\to D$ that are defined by $$\sigma_{ab}(x) = \begin{cases}
        a,\quad&x=b\\
        b, \quad& x=a\\
        x, \quad& x\neq a, b
    \end{cases}$$ for $a, b\in D$. 
\end{lemma}
\begin{proof}
    This is a reformulation of Theorem 2.1 by Conrad \cite{conrad2013generating}.
\end{proof}

\begin{lemma}
    \label{lemma:perm_to_perm}
    Let $n\in\bb N$, let $\sigma, \tau:[n]\to [n]$ be permutations and let $\psi$ be an allocation procedure. If $$\psi_{\tau(i)}(\tau(\nu)) = \psi_i(\nu)\quad\text{and}\quad\psi_{\sigma(i)}(\sigma(\nu))\quad\text{for all $i\in [n]$ and all cooperative games $\nu$},$$ then we also have $\psi_{\tau\circ\sigma(i)}(\tau\circ\sigma(\nu)) = \psi_i(\nu)$ for all $i\in [n]$ and all cooperative games $\nu$.
\end{lemma}
\begin{proof}
    Let $i\in D$ and let $\nu$ be a cooperative game with players $D$. We have that $$\psi_{\tau\circ\sigma(i)}(\tau\circ\sigma(\nu)) = \psi_{\tau(\sigma(i))}(\tau(\sigma(\nu)) = \psi_{\sigma(i)}(\sigma(\nu)) = \psi_i(\nu).$$
\end{proof}
With these lemmas, we can now prove the equivalence.
\begin{lemma}
    Let $\psi$ be an allocation procedure that satisfies strong monotonicity. We have $$\text{$\psi$ is symmetric}\quad\iff\quad\text{$\psi$ is newly symmetric.}$$
\end{lemma}
\begin{proof}
    "$\Rightarrow$": We first assume that $\psi$ is symmetric. Let $n\in\bb N$, let $\nu:\mathcal P([n])\to\bb R$ be a cooperative game and let $i, j\in [n]$. Suppose that
    $$\nu(S\cup\set i) = \nu(S\cup\set j)\quad\text{for all $S\subseteq [n]\setminus\set{i, j}$.}$$
    
    We will now prove that $\sigma_{ij}(\nu)(S) = \nu(S)$ for all $S\subseteq \mathcal A(x')$, with $\sigma_{ij}$ defined as in \autoref{lemma:perm_prod}. Let $S\subseteq[n]$. We will make a case distinction:

       \emph{Case $i, j\notin S$}: We have that $\sigma_{ij}(S) = S$ and therefore $$\sigma_{ij}(\nu)(S) = \nu(S).$$
\\\\
        \emph{Case $i\in S$ and $j\notin S$}: We now have that $$\sigma_{ij}(\nu)(S) = \nu(\sigma_{ij}((S\setminus\set i)\cup\set i)) = \nu((S\setminus \set i)\cup \set j) =  \nu((S\setminus\set i)\cup\set i) = \nu(S),$$ where in these second to last step, we use the assumption that $\nu(S\cup\set i) = \nu(S\cup\set j)$.
        \\\\
        \emph{Case $i\notin S$ and $j\in S$}: The proof of this case is analogous to the previous case.
       \\\\
        \emph{Case $i, j\in S$}: We get that $$\sigma_{ij}(\nu)(S)= \nu(\sigma_{ij}((S\setminus\set{i, j})\cup\set{i, j})) = \nu((S\setminus\set{i,j})\cup\set{j, i}) = \nu(S).$$

    We can now conclude that $\nu(S) = \nu(\sigma_{ij}(S))$ for all $S\subseteq \mathcal A(x')$. Since $\psi$ is symmetric, we get that $\psi_i(\nu) = \psi_{\sigma_{ij}(i)}(\sigma_{ij}(\nu)) = \psi_j(\sigma_{ij}(\nu))=\psi_j(\nu)$. We can now conclude that $\psi$ is newly symmetric.
    \\\\
    "$\Leftarrow$": Now suppose that $\psi$ is newly symmetric. 
    Let $\nu$ be a cooperative game and let $i, j\in [n]$. We can now make use of an observation from Young \cite[70]{young1985monotonic} and \autoref{lemma:equiv} to get that strong monotonicity gives us the following implication. 
    \begin{quote}
    Let $\omega, \mu:\mathcal P([n])\to\bb R$ be cooperative games. If $\omega(S) - \omega(S\setminus\set i) = \mu(S) - \mu(S\setminus\set i)$ for all $S\subseteq [n]$ with $i\in S$, then $\psi_i(\omega) = \psi_i(\mu)$.
    \end{quote} 
    We can now define the following cooperative game $\xi$. $$\xi(S) = \begin{cases}
        0&\quad \text{$i, j\notin S$}\\
        \nu(S) - \nu(S\setminus\set i)&\quad \text{$i\in S$ and $j\notin S$}\\
        \nu((S\setminus\set j)\cup\set i) - \nu(S\setminus\set j)&\quad \text{$i\notin S$ and $j\in S$}\\
        \nu(S) - \nu(S\setminus\set i) + \nu(S\setminus\set j) - \nu(S\setminus\set {i, j})&\quad i, j\in S.
    \end{cases}$$
    This is a cooperative game, because $i, j\notin \emptyset$, so $\xi(\emptyset) = 0$.
    
    We will first show that $\psi_i(\nu) = \psi_i(\xi)$ using the observation made by Young. After this, we will show that $\psi_i(\xi) = \psi_j(\xi)$  via symmetry and then we will show that $\psi_{\sigma_{ij}(i)}(\sigma_{ij}(\nu)) = \psi_i(\nu)$, again using the observation made by Young.
\\\\
 $\boldsymbol{\psi_i(\nu) = \psi_i(\xi)}$: Let $S\subseteq [n]$ with $i\in S$. If $j\notin S$, we have that $$\xi(S) - \xi(S\setminus\set i) = \nu(S) - \nu(S\setminus\set i) - 0=\nu(S) - \nu(S\setminus\set i).$$ If $j\in S$, we have that \begin{align*}\xi(S) - \xi(S\setminus\set i) &= \nu(S) - \nu(S\setminus\set i) + \nu(S\setminus\set j) - \nu(S\setminus\set{i, j}) - \nu(S\setminus\set j) + \nu(S\setminus\set{i, j}) \\ &= \nu(S) - \nu(S\setminus\set i).\end{align*}
    We can conclude that for all $S\subseteq [n]$ with $i\in S$ that $\nu(S) - \nu(S\setminus\set i) = \xi(S) - \xi(S\setminus\set i)$, so we can conclude that $\psi_i(\nu) = \psi_i(\xi)$. 
    \\\\ 
   $\boldsymbol{\psi_i(\xi) = \psi_j(\xi)}$: Now let $S\subseteq [n]\setminus\set{i, j}$. We have that $$\xi(S\cup\set i) = \nu(S\cup\set i) - \nu(S) = \xi(S\cup\set j).$$ Because $\psi$ is newly symmetric, we can conclude that $\psi_i(\xi) = \psi_j(\xi)$. 
\\\\
    $\boldsymbol{\psi_j(\xi) = \psi_j(\sigma_{ij}(\nu))}$: Now take any $S\subseteq [n]$ with $j\in S$. If $i\notin S$, we have $$\xi(S) - \xi(S\setminus\set j) = \nu((S\setminus\set j)\cup\set i) - \nu(S\setminus\set j) = \sigma_{ij}(\nu)(S) - \sigma_{ij}(\nu)(S\setminus\set j).$$ If $i\in S$, then we have \begin{align*}
        \xi(S) - \xi(S\setminus\set j) &= \nu(S) - \nu(S\setminus\set i) + \nu(S\setminus\set j) - \nu(S\setminus\set{i, j}) - \nu(S\setminus\set j) + \nu(S\setminus\set{i, j})\\
        &= \nu(S)-\nu(S\setminus\set i)\\
        &= \sigma_{ij}(\nu)(S) - \sigma_{ij}(\nu)(S\setminus\set j).
    \end{align*}
    We now have that for all $S\subseteq [n]$ with $j\in S$ that $$\xi(S) - \xi(S\setminus\set j) = \sigma_{ij}(\nu)(S) - \sigma_{ij}(\nu)(S\setminus\set j),$$ so we can conclude that $\psi_j(\xi) = \psi_j(\sigma_{ij}(\nu))$.

    From the above, we can conclude that $$\psi_i(\nu) = \psi_i(\xi) = \psi_j(\xi) = \psi_j(\sigma_{ij}(\nu)) = \psi_{\sigma_{ij}(i)}(\sigma_{ij}(\nu)).$$

    From \autoref{lemma:perm_prod}, we get that we can write every permutation $\sigma:[n]\to [n]$ as the composition of a finite number of functions of the form $\sigma_{ij}$ for $i, j\in [n]$. Because of this and \autoref{lemma:perm_to_perm}, we conclude that for all permutations $\sigma: [n]\to [n] $, $\psi_{\sigma(i)}(\sigma(\nu)) = \psi_i(\nu)$. We conclude that $\psi$ is symmetric. 
\end{proof}
\chapter{Convex functions}
A very important concept in the study of optimization problems is convexity. Let us first recall the definition of convexity. We will first recall the definition of a convex set.
\begin{definition}[Convex set]
    Let $n\in\bb N$ and let $S\subseteq\bb R^n$. We call $S$ convex if for all $x, y\in S$ and all $t\in[0, 1]$ $$tx + (1-t)y\in S.$$
\end{definition}
We will now also recall the definition of a convex function.
\begin{definition}[Convex function]
\label{definition:convex}
    Let $S\subseteq\bb R^n$ be a convex set. A function $f:S\to\bb R$ is called convex if for all $x, y\in\bb R^d$ and all $t\in[0, 1]$ $$f(tx + (1-t)y)\leq tf(x) + (1-t)f(y).$$
\end{definition}

While it is useful to know if a function is convex, it is sometimes hard to check directly from the definition. We will therefore look at a theorem that makes it easier to determine whether a function is convex.

\begin{theorem}
    \label{theorem:convex_second_derivative}
    Let $S\subseteq\bb R^n$ be nonempty, convex and open. Let $f:S\to\bb R$ be a function that is twice differentiable on $S$. Then $f$ is convex if and only if $\frac{\partial^2}{\partial x\partial x^T}f(x)$ is positive semi-definite.
\end{theorem}
\begin{proof}
    See the proof of Theorem 4.5 from \cite{rockafellar1997convex}.
\end{proof}

For this thesis, we want to know if composition of a convex function with an affine function preserves convexity. For clarification, a function $g:\bb R^m\to\bb R^n$ is called affine if there exist $A\in\bb R^{n\times m}$ and $b\in\bb R^n$ such that $g(x) = Ax + b$.

\begin{lemma}
\label{lemma:convex+affine}
    Let $S\subseteq\bb R^n$ be convex and let $f: S\to\bb R$ be a convex function and let $g:\bb R^m\to\bb S$ be an affine function. Then the composition $f\circ g$ is convex. 
\end{lemma}
\begin{proof}
    Since $g$ is an affine function, there exists $A\in\bb R^{n\times m}$ and $b\in\bb R^n$ such that $g(x) = Ax + b$. Let $t\in[0, 1]$ and let $x, y\in\bb R^m$. We have that 
    \begin{align*}
        (f\circ g)(tx + (1-t)y) &= f(A(tx + (1-t)y) + b)\\
        &= f(Atx + (1-t)Ay + b)\\
        &= f(Atx + tb + (1-t)Ay + (1-t)b)\\
        &\leq tf(Ax + b) + (1-t)f(Ay + b)\\
        &= t(f\circ g)(x) + (1-t)(f\circ g)(y).
    \end{align*}
    With this, we have proven that $f\circ g$ is convex.
\end{proof}

\begin{theorem}
\label{theorem:convex_inequality}
    Let $S\subseteq\bb R^n$ be convex and let $f: S\to\bb R$ be a convex function that is differentiable. Then for all $x, y\in S$ we have $$f(y)\geq f(x) + \nabla f(x)^T(y-x).$$
\end{theorem}
\begin{proof}
    This proof is based on the lecture notes from \cite{convex_lecture_notes}.
    Let $x, y\in S\subseteq\bb R^n$ and let $f:S\to\bb R$ be a convex function. From the definition of convexity, we have that for $t\in[0, 1]$ that $$(1-t)f(x) + tf(y)\geq f((1-t)x + ty).$$ Rewriting this equation gives us that $$f(y)\geq f(x) + \frac{f(x - t(y-x)) - f(x)}{t} .$$Now letting $t\downarrow0$ gives that $$f(y)\geq f(x) + \nabla f(x)(y-x).$$
\end{proof}

The importance of this theorem is demonstrated in the following corollary. 

\begin{corollary}
\label{corollary:convex_minimum}
    Let $S\subseteq\bb R^n$ and let $f:S\to\bb R$ be twice differentiable and convex. Suppose that for some $x\in S$ we have that $f$ takes on an extremum at $x$. Then $f(x)$ is a minimum of $f$.
\end{corollary}
\begin{proof}
    Since $x$ is an extremum of $f$, we have that $\nabla f(x) = 0$. Now \autoref{theorem:convex_inequality} gives us that for all $y\in S$, we have $$f(y)\geq f(x) + \nabla f(x)^T(y-x) = f(x).$$ This proves that $f(x)$ is a minimum of $f$.
\end{proof}

\chapter{Applicable linear algebra}
This section gives lemmas, from linear algebra, that are useful for proving \autoref{thm:2}. 
\begin{lemma}
    \label{lemma:outer_product_sum}
    Let $m, n\in\bb N$. Let $X\in \bb R^{m\times n}$ and $W\in\bb R^{m\times m}$ a diagonal matrix. Denote $r_i$ to be the $i$'th row of $X$. We get that $$X^TWX = \sum_{k=1}^mW_{kk}(r_kr_k^T).$$
\end{lemma}
\begin{proof}
    We will prove this by using the definition of matrix multiplication. We get that \begin{align*}[X^TWX]_{ij} &= \sum_{k=1}^m\sum_{\ell=1}^m X_{ik}^TW_{k\ell}X_{\ell j} \\&= \sum_{k=1}^mX_{ki}W_{kk}X_{kj}\\& = \sum_{k=1}^{m}W_{kk}(r_k)_i(r_k)_j \\&= \sum_{k=1}^mW_{kk}[r_kr_k^T]_{ij} \\&= \bric{\sum_{k=1}^mW_{kk}r_kr_k^T}_{ij}.\end{align*} Since the $X^TWX$ and $\sum_{k=1}^mW_{kk}(r_kr_k^T)$ are the same on every index, we can conclude that $$X^TWX = \sum_{k=1}^mW_{kk}(r_kr_k^T).$$
\end{proof}
\begin{lemma}
    \label{lemma:sum_binary_vectors}
    Let $n\in\bb N$ and let $s\in\set{1, \dots, n}$. The following equality holds $$\sum_{\substack{z\in\set{0, 1}^n\\|z|=s}}zz^T = \binom{n-2}{s-1}I + \binom{n-2}{s-2}J,$$ where $I, J\in\bb R^{n\times n}$ with $I$ the identity matrix and $J$ the matrix with only ones. We use the convention that for $k\in\bb N$ and $n\in\bb N$ such that $k<0$ or $n>k$, that $\binom{n}{k} = 0$.
\end{lemma}
\begin{proof}
First suppose that $s>1$.
    We now have that for $z\in\set{0, 1}^n$, $[zz^T]_{ij} = 1$ if and only if $z_i = 1$ and $z_j=1$. This means that for $i, j\in\set{1, \dots, n}$ with $i\neq j$ that $\sum_{\substack{z\in\set{0, 1}^n\\|z|=s}}[zz^T]_{ij}$ is equal to $\#\set{z\in\set{0, 1}^n: z_i = z_j = 1, |z|=s}$. This is equal to $\binom{n-2}{s-2}$, because we need to fill $n-2$ spots with $s-2$ ones. 

    Now suppose that $i=j$. Then we have that $[zz^T]_{ij}=1$ if and only if $z_i=1$. This means that, through similar logic as before, $$\sum_{\substack{z\in\set{0, 1}^n\\|z|=s}}[zz^T]_{ii} = \#\set{z\in\set{0, 1}^n: z_i=1, |z|=s} = \binom{n-1}{s-1},$$ because we need to fill $n-1$ spots in a vector with $s-1$ ones.

    We can now conclude that $$\sum_{\substack{z\in\set{0, 1}^n\\|z|=s}}zz^T = \brac{\binom{n-1}{s-1} - \binom{n-2}{s-2}} I  + \binom{n-2}{s-2}J = \binom{n-2}{s-1}I + \binom{n-2}{s-2}J.$$

    Now suppose that $s=1$. We now find that $$\sum_{\substack{z\in\set{0, 1}^n\\|z|=s}}zz^T = I = \binom{n-2}{0}I + \binom{n-2}{-1}J = \binom{n-2}{s-1}I + \binom{n-2}{s-2}J.$$
\end{proof}

\begin{lemma}[Sherman-Morrison-Woodbury formula]
    \label{lemma:sherman-morrison-woodbury}
    Let $n\in\bb N$ and let $A\in\bb R^{n\times n}$ be invertible. Let $u, v\in\bb R^n$. If $1 + v^TA^{-1}u\neq 0$, then $A + uv^T$ is invertable with $$(A + uv^T)^{-1} = A^{-1} - \frac{A^{-1}uv^TA^{-1}}{1 + v^TA^{-1}u}.$$
\end{lemma}
\begin{proof}
    This proof is given on page 66 from \cite{wg1994numerical}.
\end{proof}
\end{document}